\documentclass[journal]{IEEEtran}
\usepackage{graphics}
\usepackage{epsfig}
\usepackage{times}
\usepackage{amsmath,amsthm}
\usepackage{amssymb}
\usepackage{color}
\usepackage{times}
\usepackage{graphicx}
\usepackage{algorithmicx}
\usepackage[ruled]{algorithm}
\usepackage{algpseudocode}
\usepackage{subfigure}
\usepackage{setspace}
\usepackage{graphicx}
\usepackage{subfigure}
\usepackage{url}
\usepackage{colortbl}
\usepackage{caption}
\usepackage{epstopdf}

\newtheorem{proposition}{Proposition}
\ifCLASSOPTIONcompsoc
\else
\fi
\ifCLASSINFOpdf
\else
\fi
\begin{document}

\title{A Labeled Random Finite Set Online Multi-Object Tracker for Video Data}
\author{Du Yong Kim, Ba-Ngu Vo,~\IEEEmembership{Member,~IEEE,} and~Ba-Tuong
Vo,~\IEEEmembership{Member,~IEEE}}
\maketitle

%\maketitle



\markboth{Journal of \LaTeX\ Class Files,~Vol.~x, No.~x, x~xxxx}{Shell
\MakeLowercase{\textit{et al.}}: Bare Demo of IEEEtran.cls for Computer
Society Journals}

\IEEEcompsoctitleabstractindextext{\begin{abstract}
This paper proposes an online multi-object tracking algorithm for image observations using a top-down Bayesian formulation that 
seamlessly integrates state estimation, track management, clutter rejection, occlusion and mis-detection handling 
into a single recursion. This is achieved by modeling the multi-object state as labeled random finite set and using the 
Bayes recursion to propagate the multi-object filtering density forward in time. The proposed filter updates tracks with 
detections but switches to image data when mis-detection occurs, thereby exploiting the efficiency of 
detection data and the accuracy of image data. Furthermore the labeled random finite set framework enables the 
incorporation of prior knowledge that mis-detections in the middle of the scene are likely to be due to 
occlusions. Such prior knowledge can be exploited to improve occlusion handling, especially long occlusions that 
can lead to premature track termination in on-line multi-object tracking. Tracking performance is compared to 
state-of-the-art algorithms on synthetic data and well-known benchmark video datasets.

\end{abstract}

} % make the title area
% is not selected <OR> if conference mode is selected - because compsoc
% conference papers position the abstract like regular (non-compsoc)
\IEEEdisplaynotcompsoctitleabstractindextext

\IEEEpeerreviewmaketitle

\section{Introduction}

In a multiple object setting, not only do the states of the objects vary
with time, but the number of objects also changes due to objects appearing
and disappearing. In this work, we consider the problem of jointly
estimating the time-varying number of objects and their trajectories from a
stream of noisy images. In particular, we are interested in multi-object
tracking (MOT) solutions that compute estimates at a given time using only
data up to that time. These so-called online solutions are better suited for
time-critical applications.

A critical function of a multi-object tracker is track management, which
concerns track initiation/termination and track labeling or identifying
trajectories of individual objects. Track management is more challenging for
online algorithms than for batch algorithms. Usually, track
initiation/termination in online MOT algorithms is performed by examining
consecutive detections in time \cite{Breitenstein11}, \cite{BPF}. However,
false positives generated by the background, compounded by false negatives
from object occlusions and mis-detections, can result in false tracks and
lost tracks, especially in online algorithms. False negatives also cause
track fragmentation in batch algorithms as reported in \cite{DAT08}, \cite%
{Berclaz}, \cite{Milan_2014}. With the exception of the recent network flow 
\cite{Afshin_CVPR15} techniques, track labels are assigned upon track
initiation, and maintained over time until termination. An online
multi-object Bayesian filter that provides systematic track labeling using
labeled random finite set (RFS) was proposed in \cite{VV13}.\newline

In most video MOT approaches, each image in the data sequence is compressed
into a set of detections before a \emph{filtering} operation is applied to
keep track of the objects (including undetected ones) \cite{BPF, Maggio08}.
Typically, in the filtering module, motion correspondence or data
association is first determined followed by the application of standard
filtering techniques such as Kalman or sequential Monte Carlo~\cite{Maggio08}%
. The main advantage of performing detection before filtering is the
computational efficiency in the compression of images into relevant
detections. The main disadvantage is the loss of information, in addition to
mis-detection and false alarms, especially in low signal to noise ratio
(SNR) applications.

Track-before-detect (TBD) is an alternative approach, which by-passes the
detection module and exploits the spatio-temporal information directly from
the image sequence. The TBD methodology is often required in tracking
applications for low SNR image data \cite{TBD_survey}, \cite{Vo2010}, \cite{PapiKim15}, \cite{Papi_etal15}, \cite{Santosh13}, \cite{Santosh15}. In visual tracking
applications, perhaps the most well-known TBD MOT algorithm is BraMBLe~\cite%
{Bramble}. Other visual MOT algorithms that can be categorized as TBD
include \cite{Perez02}, \cite{Color_PF} which exploit color-based
observation models, \cite{Mix_PF03}, \cite{BPF}, \cite{Branko_IVC}, which
exploit multi-modality of distributions, and \cite{Hoseinnezhad2012}, \cite{Hoseinnezhad2013} which
uses multi-Bernoulli random finite set models. While the TBD approach
minimizes information loss, it is computationally more expensive. A balance
between tractability and fidelity is important in the design of the
measurement model.

\indent In this paper, we present an efficient online MOT algorithm for
video data that exploits the advantages of both detection-based and TBD
approaches to improve performance while reducing the computational cost. The
proposed MOT filter updates the tracks adaptively with detections for
efficiency, or with local regions on the image to minimize information loss.
In addition it seamlessly integrates state estimation, track management,
clutter rejection, mis-detection and occlusion handling in a single Bayesian
recursion.\newline
\indent Specifically, using the RFS framework we propose a hybrid
multi-object likelihood function that accommodates both detections and image
observations, thereby generalizing the standard multi-object likelihood \cite%
{Mahler_book07} and the separable likelihood for image \cite{Vo2010}.
Further, we establish conjugacy of the Generalized Labelled Multi-Bernoulli
(GLMB) distributions with respect to the proposed likelihood function, which
then yields an analytic solution to the multi-object Bayes recursion since
the GLMB family is closed under the Chapman-Kolmogorov equation. The
proposed MOT filter exploits the efficiency of the detection-based approach
which avoids updating with the entire image, while at the same time
exploiting relevant information at the image level by using only small
regions of the image where mis-detected objects are expected. 

The labeled RFS formulation \cite{VV13} addresses state estimation, track
management, clutter rejection, mis-detection and occlusion handling, in one
single Bayes recursion. Generally, an online MOT algorithm would terminate a
track that has not been detected over several frames. In many video MOT
applications however, it is observed that away from designated exit regions
such as scene edges, the longer an object is in the scene, the less likely
it is to disappear. Intuitively, this observation can be used to delay the
termination of tracks that have been occluded over an extended period, so as
to improve occlusion handling. The use of labeled RFS in our proposed filter
provides a principled and inexpensive means to exploit this observation for
improved occlusion handling.

\indent The remainder of the paper is structured as follows. The Bayesian
filtering formulation of the MOT problem using labeled RFS is given in
Section \ref{sec:MOT}, followed by details of the proposed solution in
Section \ref{sec:GLMBMOT}. Performance evaluation of the proposed MOT filter
against state-of-the-art trackers is presented in Section \ref%
{sec:Experiment}, and concluding remarks are given in Section \ref%
{sec:Conclusion}. 
\begin{figure}[tbp]
	\begin{center}
		\includegraphics[height=.45\linewidth]{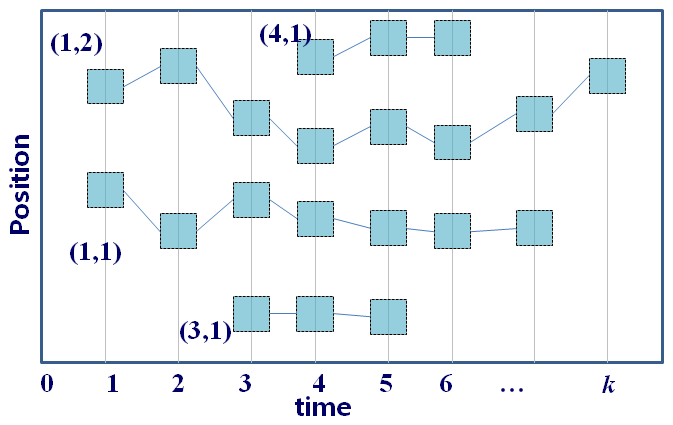} 
	\end{center}
	\caption{1D multi-object trajectories with labeling}
	\label{Fig: label}
\end{figure}

\section{Bayesian Multiple Object Tracking\label{sec:MOT}}

This section outlines the RFS framework for MOT that accommodates
uncertainty in the number of objects, the states of the objects and their
trajectories. The salient feature of this framework is that it admits direct
parallels between traditional Bayesian filtering and MOT. The modeling of
the multi-object state as an RFS in Subsection \ref{subsec:MOS} enables
Bayesian filtering concepts to be directly translated to the multi-object
case in Subsection \ref{subsec:MOF}. Subsection \ref{subsec:Occlusion}
examines the MOT problem in the presence of occlusion.

\subsection{Multi-object State\label{subsec:MOS}}

\indent To distinguish different object trajectories in a multi-object
setting, each object is assigned a unique label $\ell _{k}$ that consists of
an ordered pair $(t,i)$, where $t$ is the time of birth and $i$ is the index
of individual objects born at the same time \cite{VV13}. For example, if two
objects appear in the scene at time 1, one is assigned label (1,1) while the
other is assigned label (1,2), see Fig. \ref{Fig: label}. A trajectory or
track is the sequence of states with the same label.\newline
\indent Formally, the state of an object at time $k$ is a vector $\mathbf{x}%
_{k}=(x_{k},\ell _{k})\in \mathbb{X\times L}_{k}$, where $\mathbb{L}_{k}$
denotes the label space for objects at time $k$ (including those born prior
to $k$). Note that $\mathbb{L}_{k}$ is given by $\mathbb{B}_{k}\cup \mathbb{L%
}_{k-1}$, where $\mathbb{B}_{k}$ denotes the label space for objects born at
time $k$ (and is disjoint from $\mathbb{L}_{k-1}$). Suppose that there are $%
N_{k}$ objects at time $k$, with states $\mathbf{x}_{k,1},...,\mathbf{x}%
_{k,N_{k}}$. In the context of MOT, the collection of states, referred to as
the \emph{multi-object state}, is naturally represented as a finite set 
\begin{equation*}
\mathbf{X}_{k}=\{\mathbf{x}_{k,1},...,\mathbf{x}_{k,N_{k}}\}\in \mathcal{F}(%
\mathbb{X\times L}_{k}),
\end{equation*}%
where $\mathcal{F}(\mathbb{X\times L}_{k})$ denotes the space of finite
subsets of $\mathbb{X\times L}_{k}$. We denote cardinality (number of
elements) of $\mathbf{X}$ by $|\mathbf{X}|$ and the set of labels of $%
\mathbf{X}$, $\{\ell :(x,\ell )\in \mathbf{X}\}$, by $\mathcal{L}(\mathbf{X}%
) $. Note that since the label is unique, no two objects have the same
label, i.e. $\delta _{|\mathbf{X}|}(|\mathcal{L}(\mathbf{X})|)=1$. Hence $%
\Delta (\mathbf{X})\triangleq $ $\delta _{|\mathbf{X}|}(|\mathcal{L}(\mathbf{%
X})|)$ is called the \emph{distinct label indicator}.\newline
\indent For the rest of the paper, we follow the convention that
single-object states are represented by lower-case letters (e.g. $x$, $%
\mathbf{x}$), while multi-object states are represented by upper-case
letters (e.g. $X$, $\mathbf{X}$), symbols for labeled states and their
distributions are bold-faced to distinguish them from unlabeled ones (e.g. $%
\mathbf{x}$, $\mathbf{X}$, $\mathbf{\pi }$, etc.), and spaces are
represented by blackboard bold (e.g. $\mathbb{X}$, $\mathbb{Z}$, $\mathbb{L}$%
, $\mathbb{N}$, etc.). The list of variables $X_{m},X_{m+1},...,X_{n}$ is
abbreviated as $X_{m:n}$. We denote a generalization of the Kroneker delta
that takes arbitrary arguments such as sets, vectors, integers etc., by 
\begin{equation*}
\delta_{Y}[X]\triangleq \left\{ 
\begin{array}{l}
1,\text{ if }X=Y \\ 
0,\text{ otherwise}%
\end{array}%
\right. .
\end{equation*}%
For a given set $S$, $1_{S}(\cdot )$ denotes the indicator function of $S$,
and $\mathcal{F}(S)$ denotes the class of finite subsets of $S$. For a
finite set $X$, the multi-object exponential notation $f^{X}$ denotes the
product $\prod_{x\in X}f(x)$, with $f^{\emptyset }=1$. The inner product $%
\int f(x)g(x)dx$ is denoted by $\left\langle f,g\right\rangle $.

\subsection{Multi-object Bayes filter\label{subsec:MOF}}

From a Bayesian estimation viewpoint the multi-object state is naturally
modeled as an RFS or a simple-finite point process \cite{Mahler2003}. While
the space $\mathcal{F}(\mathbb{X\times L}_{k})$ does not inherit the
Euclidean notion of probability density, Mahler's Finite Set Statistic
(FISST) provides a suitable notion of integration/density for RFSs \cite%
{Mahler_book07, Mahler_book14}. This approach is mathematically consistent
with measure theoretic integration/density but circumvents measure theoretic
constructs \cite{Vo_SMCPHD}.\newline
\indent At time $k$, the multi-object state $\mathbf{X}_{k}$ is observed as
an image $y_{k}$. All information on the set of object trajectories
conditioned on the observation history $y_{1:k}$, is captured in the \emph{%
multi-object posterior density} 
\begin{equation*}
\boldsymbol{\pi }_{0:k}(\mathbf{X}_{0:k}|y_{1:k}\mathbf{)}\propto
\prod\limits_{j=1}^{k}g_{j}(y_{j}|\mathbf{X}_{j})\mathbf{f}_{j|j-1\!}(%
\mathbf{X}_{j}|\mathbf{X}_{j-1})\boldsymbol{\pi }_{0}(\mathbf{X}_{0}\mathbf{)%
}
\end{equation*}%
where $\boldsymbol{\pi }_{0}$ is the initial prior, $g_{j}(\cdot |\cdot )$
is the \emph{multi-object likelihood} \emph{function} at time $j$, $\mathbf{f%
}_{j|j-1}(\cdot |\cdot )$ is the \emph{\ multi-object transition density} to
time $j$. The multi-object likelihood function encapsulates the underlying
observation model while the multi-object transition density encapsulates the
underlying models for motions, births and deaths of objects. Note that track
management is incorporated into the Bayes recursion via the multi-object
state with distinct labels.\newline
\indent MCMC approximations of the posterior density have been proposed in 
\cite{Vu14,Craciun} for detection measurements and image measurements
respectively. Results on satellite imaging applications reported in \cite%
{Craciun} are very impressive. However, these techniques are still expensive
and not suitable for on-line application.\newline
\indent For real-time tracking, a more tractable alternative is the \emph{%
multi-object filtering density}, a marginal of the multi-object posterior.
For notational compactness, herein we omit the dependence on data in the
multi-object densities. The multi-object filtering density can be
recursively propagated by the \emph{multi-object Bayes filter} \cite%
{Mahler2003}, \cite{Mahler2007a} according to the following prediction and
update%
\begin{align}
\boldsymbol{\pi }_{k+1|k}(\mathbf{X}_{k+1})& =\int \mathbf{f}_{k+1|k\!}(%
\mathbf{X}_{k+1}|\mathbf{X}_{k})\boldsymbol{\pi }_{k}(\mathbf{X}_{k}\mathbf{)%
}\delta \mathbf{X}_{k},  \label{Chap_Kolmo_LRFS} \\
\boldsymbol{\pi }_{k+1}(\mathbf{X}_{k+1})& =\frac{g_{k+1}(y_{k+1}|\mathbf{X}%
_{k+1})\boldsymbol{\pi }_{k+1|k}(\mathbf{X}_{k+1})}{\int g_{k+1}(y_{k+1}|%
\mathbf{X})\boldsymbol{\pi }_{k+1|k}(\mathbf{X})\delta \mathbf{X}},
\label{Bayesupdate_LRFS}
\end{align}%
where the integral is a \emph{set integral} defined for any function $%
\mathbf{f:\mathcal{F}(}\mathbb{X}\mathcal{\times }\mathbb{L}_{k}\mathbb{)}%
\rightarrow \mathbb{R}$ by 
\begin{equation*}
\int \mathbf{f}(\mathbf{X})\delta \mathbf{X}=\sum_{i=0}^{\infty }\frac{1}{i!}%
\int \mathbf{f}(\{\mathbf{x}_{1},...,\mathbf{x}_{i}\})d(\mathbf{x}_{1},...,%
\mathbf{x}_{i}).
\end{equation*}

Bayes optimal multi-object estimators can be formulated by minimizing the
Bayes risk with ordinary integrals replaced by set integrals as in \cite%
{Mahler_book14}. One such estimator is the marginal multi-object estimator 
\cite{Mahler_book07}.\newline
\indent A generic particle implementation of the Bayes multi-object filter (%
\ref{Chap_Kolmo_LRFS})-(\ref{Bayesupdate_LRFS}) was proposed in \cite%
{Vo_SMCPHD} and applied to labeled multi-object states in \cite{PapiKim15}.
The \emph{Generalized labeled Multi-Bernoulli} (GLMB) filter is an analytic
solution to the Bayes multi-object filter, under the standard multi-object
dynamic and observation models \cite{VV13}.

\subsubsection{Standard multi-object dynamic model}

Given the multi-object state $\mathbf{X}_{k}$ (at time $k$), each state $%
(x_{k},\ell _{k})\in \mathbf{X}_{k}$ either survives with probability $%
P_{S,k}(x_{k},\ell _{k})$ and evolves to a new state $(x_{k+1},\ell _{k+1})$
(at time $k+1$) with probability density $f_{k+1|k}(x_{k+1}|x_{k},\ell
_{k})\delta _{\ell _{k}}[\ell _{k+1}]$ or dies with probability $%
1-P_{S,k}(x_{k},\ell _{k})$. The set $\mathbf{B}_{k+1}$ of new objects born
at time $k+1$ is distributed according to the labeled multi-Bernoulli (LMB)%
\begin{equation*}
\Delta (\mathbf{B}_{k+1})\omega _{B,k+1}(\mathcal{L(}\mathbf{B}%
_{k+1}))p_{B,k+1}^{\mathbf{B}_{k+1}},
\end{equation*}%
where $\omega _{B,k+1}(L)=\left[ 1_{\mathbb{B}_{k+1}}\,r_{B,k+1}\right] ^{L}%
\left[ 1-r_{B,k+1}\right] ^{\mathbb{B}_{k+1}-L}$, $r_{B,k+1}(\ell )$ is the
probability that a new object with label $\ell $ is born, and $%
p_{B,k+1}(\cdot ,\ell )$\ is the distribution of its kinematic state \cite%
{VV13}. The multi-object state $\mathbf{X}_{k+1}$ (at time $k+1$) is the
superposition of surviving objects and new born objects. It is assumed that,
conditional on $\mathbf{X}_{k}$, \ objects move, appear and die
independently of each other. The expression for the multi-object transition
density $\mathbf{f}_{k+1|k}$ can be found in \cite{VV13, VVP14}. The
standard multi-object dynamic model enables the Bayes multi-object filter to
address motion, births and deaths of objects.

\subsubsection{Standard multi-object observation model}

In most applications a designated detection operation $D$ is applied to $%
y_{k}$ resulting in a set of points 
\begin{equation}
Z_{k}=D(y_{k}) \in \mathbb{Z}.  \label{point_obs}
\end{equation}
Since the detection process is not perfect,
false positives and false negatives are inevitable. Hence only a subset of $%
Z_{k}$ correspond to some objects in the scene (not all objects are
detected) while the remainder are false positives. The most popular
detection-based observation model is described in the following.\newline
\indent For a given multi-object state $\mathbf{X}_{k}$, each $(x,\ell )\in 
\mathbf{X}_{k}$ is either detected with probability $P_{D,k}(x,\ell )$ and
generates a detection $z\in Z_{k}$ with likelihood $g_{D,k}(z|x,\ell )$ or
missed with probability $1-P_{D,k}(x,\ell )$. The \emph{multi-object
observation }$Z_{k}$ is the superposition of the observations from detected
objects and Poisson clutter with intensity $\kappa_{k}$. The ratio 
\begin{equation}
\sigma_{D,k}(z|x,\ell )\triangleq \frac{g_{D,k}(z|x,\ell )}{\kappa _{k}(z)}
\label{eq:SNRD}
\end{equation}%
can be interpreted as the detection signal to noise ratio (SNR).\newline
\indent Assuming that, conditional on $\mathbf{X}_{k}$, detections are
independent of each other and clutter, the multi-object likelihood function
is given by \cite{Mahler_book07}, \cite{VV13, VVP14} 
\begin{equation}
g_{k}(y_{k}|\mathbf{X}_{k})\propto \sum_{\theta \in \Theta _{k}(\mathcal{L(}%
\mathbf{X}_{k}))}\prod\limits_{(x,\ell )\in \mathbf{X}_{k}}\psi
_{D(y_{k})}^{(\theta (\ell ))}(x,\ell )  \label{eq:RFSmeaslikelihood0}
\end{equation}%
where: $\Theta _{k}(I)$ is the set of \emph{positive 1-1} maps $\theta
:I\rightarrow \{0$:$|Z_{k}|\}$, i.e. maps such that \emph{no two distinct
arguments are mapped to the same positive value}; and 
\begin{equation}
\psi _{\!\{z_{1:M}\}\!}^{(j)}(x,\ell )=\left\{ \!\!%
\begin{array}{ll}
P_{\!D,k}(x,\ell )\sigma _{\!D,k}(_{\!}z_{j\!}|x,\ell ), & \!\!\text{if }j=1%
\text{:}M \\ 
1-P_{\!D,k}(x,\ell ), & \!\!\text{if }j=0%
\end{array}%
\right. \!\!.  \label{eq:PropConj5}
\end{equation}%
The map $\theta $ specifies which objects generated which detections, i.e.
object $\ell $ generates detection $z_{\theta (\ell )}\in Z_{k}$, with
undetected objects assigned to $0$. The positive 1-1 property means that $%
\theta $ is 1-1 on $\{\ell :\theta (\ell )>0\}$, the set of labels that are
assigned positive values, and ensures that any detection in $Z_{k}$ is
assigned to at most one object.\newline
\indent The standard multi-object observation model enables the Bayes
multi-object filter to address mis-detection and false detection, but not
occlusion. It assumes that each object is detected independently from each
other, and that a detection cannot be assigned to more than one object. This
is clearly not valid in occlusions.

\subsection{Bayes Optimal Occlusion Handing\label{subsec:Occlusion}}

\begin{figure}[tbp]
	\begin{center}
		\includegraphics[height=.35\linewidth]{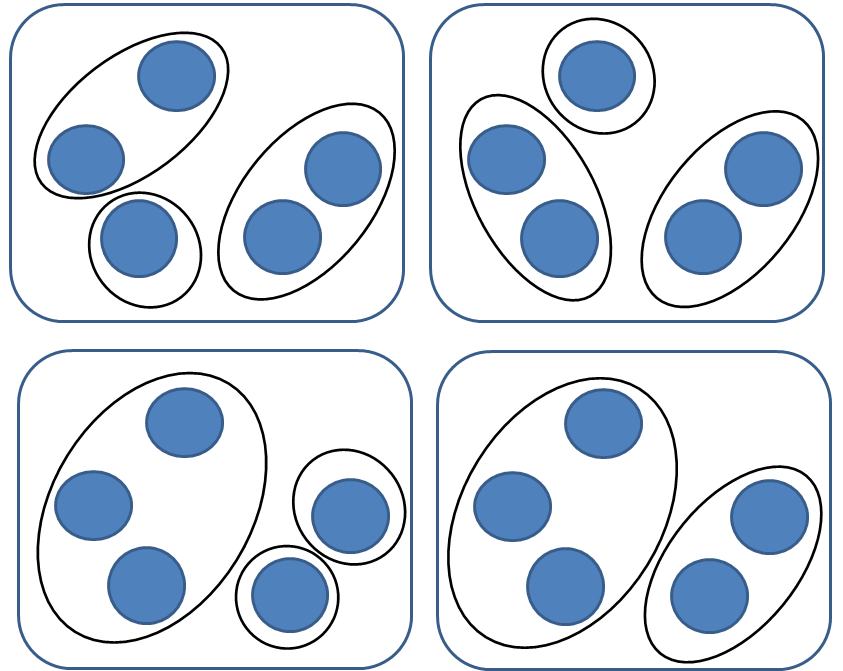} 
		\includegraphics[height=.35\linewidth]{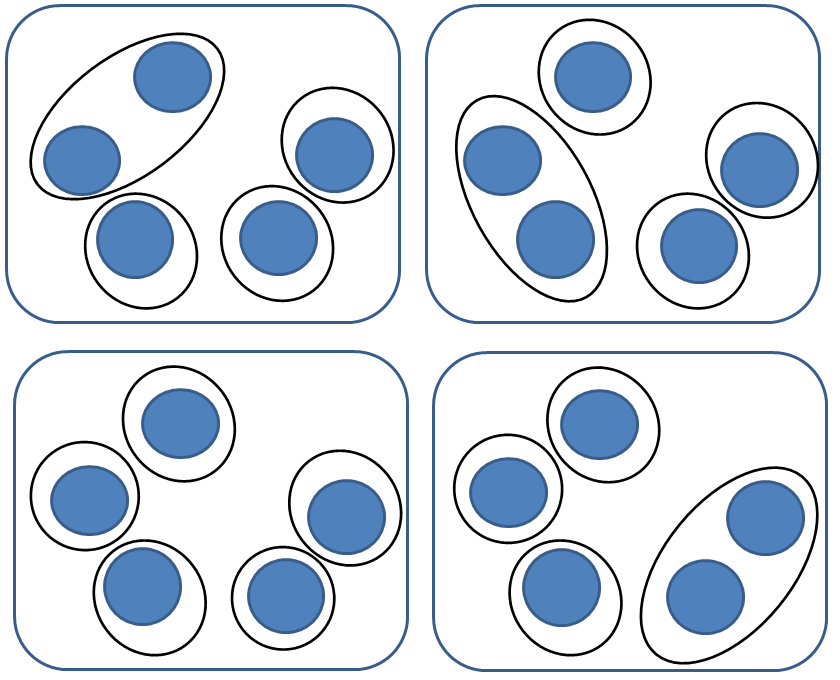} 
	\end{center}
	\caption{Examples of partitions for five objects}
	\label{Fig: partition}
\end{figure}

By relaxing the assumption that each object is independently detected, a
multi-object observation model that explicitly addresses occlusion (as well
as mis-detections and false positives) was proposed in \cite{Beard15}. The
main difference between this so-called \emph{merged-measurement} model and
the standard model is the idea that each group of objects (instead of each
object) in the multi-object state generates at most one detection \cite%
{Beard15}. Fig. \ref{Fig: partition} shows various partitions or groupings
of a multi-object state with five objects.\newline
\indent A \emph{partition} $\mathcal{U}_{\mathbf{X}}$ of a finite set $%
\mathbf{X}$ is a collection of mutually exclusive subsets of $\mathbf{X}$,
whose union is $\mathbf{X}$. The collection of all partitions of $\mathbf{X}$
is denoted by $\mathcal{P}(\mathbf{X})$.\ It is assumed that given a
partition $\mathcal{U}_{\mathbf{X}}$, each group $\mathbf{Y}\in \mathcal{U}_{%
\mathbf{X}} $ generates at most one detection with probability $P_{D,k}(%
\mathbf{Y})$, independent of other groups, and that conditional on detection
generates $z$ with likelihood $g_{D,k}(z|\mathbf{Y})$.\newline
\indent Let $\mathcal{L}(\mathcal{U}_{\mathbf{X}})$ denote the collection of
labels of the partition $\mathcal{U}_{\mathbf{X}}$, i.e. $\mathcal{L}(%
\mathcal{U}_{\mathbf{X}})\triangleq \{\mathcal{L}(\mathbf{Y}):\mathbf{Y}\in 
\mathcal{U}_{\mathbf{X}}\}$ (note that $\mathcal{L}(\mathcal{U}_{\mathbf{X}})
$ forms a partition of $\mathcal{L}(\mathbf{X})$). Let $\Xi _{k}(\mathcal{L}(%
\mathcal{U}_{\mathbf{X}})\mathcal{)}$ denote the class of all positive 1-1
mappings $\vartheta :\mathcal{L}(\mathcal{U}_{\mathbf{X}})\rightarrow \{0$:$%
|Z_{k}|\}$. Then, the likelihood that a given partition $\mathcal{U}_{%
\mathbf{X}}$ of a multi-object state $\mathbf{X}$, generates the detection
set $Z_{k}$ is 
\begin{equation}
\sum\limits_{\vartheta \in \Xi _{k}(\mathcal{L}(\mathcal{U}_{\mathbf{X}})%
\mathcal{)}}\prod\limits_{\mathbf{Y}\in \mathcal{U}_{\mathbf{X}}}\psi
_{Z_{k}}^{(\vartheta (\mathcal{L(}\mathbf{Y}\mathcal{)}))}(\mathbf{Y})
\label{Merged_likelihood2}
\end{equation}%
where%
\begin{equation*}
\psi _{\!\{z_{1:M}\}\!}^{(j)}(\mathbf{Y})=\left\{ \!\!%
\begin{array}{ll}
P_{\!D,k}(\mathbf{Y})\sigma _{\!D,k}(z_{j}|\mathbf{Y}), & \!\!\text{if }j=1%
\text{:}M \\ 
1-P_{\!D,k}(\mathbf{Y}), & \!\!\text{if }j=0%
\end{array}%
\right. ,
\end{equation*}%
with $\sigma _{D,k}(z_{j}|\mathbf{Y})=g_{D,k}(_{\!}z_{j\!}|\mathbf{Y}%
)/\kappa _{k}(z_{j})$ denoting the detection SNR for group $\mathbf{Y}$. The
merged-measurement likelihood function is obtained by summing the group
likelihoods\ (\ref{Merged_likelihood2}) over all partitions of $\mathbf{X}$ 
\cite{Beard15}:%
\begin{equation*}
g_{k}(y_{k}|\mathbf{X})=\sum_{\mathcal{U}_{\mathbf{X}}\in \mathcal{P}(%
\mathbf{X})}\sum\limits_{\vartheta \in \Xi _{k}(\mathcal{L}(\mathcal{U}_{%
\mathbf{X}})\mathcal{)}}\prod\limits_{\mathbf{Y}\in \mathcal{U}_{\mathbf{X}%
}}\!\psi _{Z_{k}}^{(\vartheta (\mathcal{L(}\mathbf{Y}\mathcal{)}))}(\mathbf{Y%
}).
\end{equation*}

The multi-object filter (\ref{Chap_Kolmo_LRFS})-(\ref{Bayesupdate_LRFS})
with merged-measurement likelihood is Bayes optimal in the sense that the
filtering density contains all information on the current multi-object state
in the presence of false positives, mis-detections and occlusions.
Unfortunately, this filter is numerically intractable due to the sum over
all partitions of the multi-object state in the merged-measurement
likelihood. At present, there is no polynomial time technique for truncating
sums over partitions. Moreover, given a partition, computations involving
the joint detection probability $P_{D,k}(\mathbf{Y})$, joint likelihood $%
g_{D,k}(z|\mathbf{Y})$ and associated joint densities are intractable except
for scenarios with a few objects.\newline
\indent A GLMB approximation that reduces the number of partitions using the
cluster structure of the predicted multi-object state and the sensor's
resolution capabilities was proposed in \cite{Beard15}. Also, computation of
joint densities are approximated by products of independent densities that
minimise the Kullback-Leibler divergence \cite{Papi_etal15}. Case studies on
MOT with bearings only measurements shows very good tracking performance.
Nonetheless, at present, this filter is still computationally demanding and
therefore not suitable for online MOT with image data.

\section{GLMB filter for tracking with image data\label{sec:GLMBMOT}}

The GLMB filter (with the standard measurement likelihood) is a suitable
candidate for online MOT \cite{VVP14, VVH}. However, it is neither designed
to handle occlusion nor image data. Even though occluded objects share the
observations of the occluding objects, this situation is not permitted in
the standard multi-object likelihood. Consequently, uncertainties in the
states of occluded objects grow, while their existence probabilities quickly
diminish to zero, leading to possible hi-jacking, and premature track
termination in longer occlusions.\newline
\indent This section proposes an efficient \emph{GLMB filter} that exploits
information from image data and addresses false positives, mis-detections
and occlusions. Subsection \ref{subsec:Likelihood} extends the standard
observation model to allow occluded objects to share observations at the
image level while Subsection \ref{subsec:Death} incorporates, into the death
model, domain knowledge that mis-detected tracks with long durations are
unlikely to disappear. The GLMB filter for image data and an efficient
implementation are then described in Subsections \ref{subsec:VGLMB} and \ref%
{subsec:Implementation}.

\subsection{Hybrid Multi-Object Measurement Likelihood\label%
{subsec:Likelihood}}

\begin{figure}[tbp]
	\begin{center}
		\includegraphics[height=.35\linewidth]{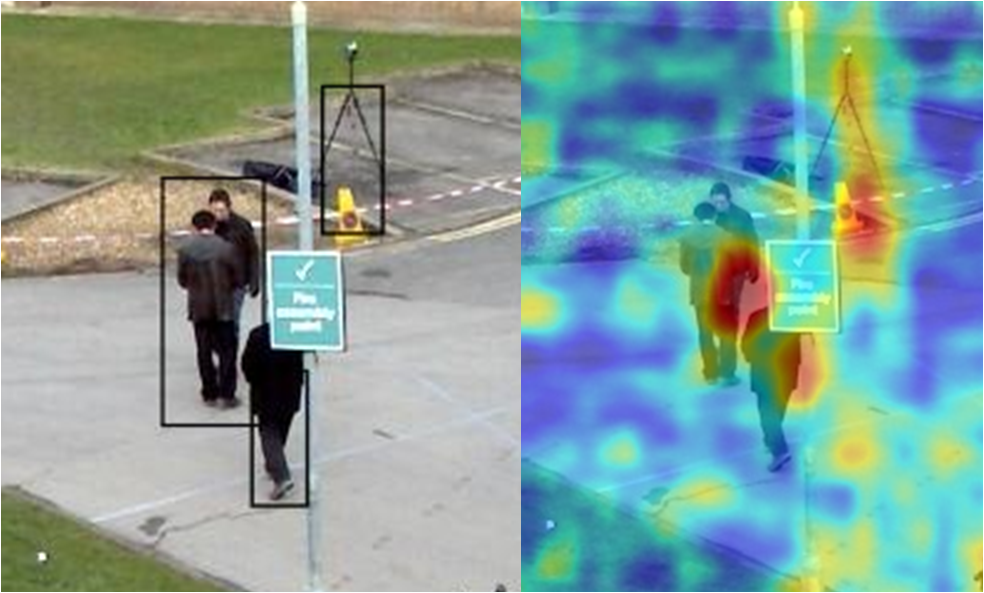} 
	\end{center}
	\caption{Example of shared measurement in occlusion}
	\label{Fig: Shared_measurement}
\end{figure}

While the detection set $Z_{k}$ is an efficient compression of the image
observation $y_{k}$, mis-detected (including occluded) objects will not be
updated by the filter. On the other hand the uncompressed observation $y_{k}$
contains relevant information about all objects, but updating with $y_{k}$
is computationally expensive. Conceptually, we can have the best of both
worlds by updating detected objects with the associated detections and
mis-detected objects with the image observations localised to regions where
these objects are expected. More importantly, this strategy exploits the
fact that occluded objects share measurements with the objects occluding
them as illustrated in Fig. \ref{Fig: Shared_measurement}.\newline
\indent A hybrid tractable multi-object likelihood function that
accommodates both detection and image observations can be obtained as
follows. For tractability, it is assumed that each object generates
observation independently from each other (similar to the standard
observation model).\newline
\indent Given an object with state $(x,\ell )$ the likelihood of observing
the local image $T(y_{k})$ (some transformation of the image $y_{k}$) is $%
g_{T,k}(T(y_{k})|x,\ell )$. On the other hand, given that there are no
objects, the likelihood of observing $T(y_{k})$ is $g_{T,k}(T(y_{k})|%
\emptyset )$. The ratio%
\begin{equation}
\sigma _{T,k}(T(y_{k})|x,\ell )\triangleq \frac{g_{T,k}(T(y_{k})|x,\ell )}{%
g_{T,k}(T(y_{k})|\emptyset )}  \label{eq:SNRV}
\end{equation}%
can be interpreted as the image SNR (c.f. detection SNR (\ref{eq:SNRD})).%
\newline
\indent For a given association map $\theta $ in the likelihood function (%
\ref{eq:RFSmeaslikelihood0}), an object with state $(x,\ell )$ is
mis-detected if $\theta (\ell )=0$, in which case the value of $\psi
_{Z_{k_{\!}}}^{(\theta (\ell ))}(x,\ell )$ is $1-P_{D,k}(x,\ell )$, the
probability of a miss. Consequently, after the Bayes update, track $\ell $
has no dependence on the observation $y_{k}$. In order for track $\ell $ to
be updated with the local image $T(y_{k})$, the value of $\psi
_{D(y_{k})}^{(\theta (\ell ))}(x,\ell )$ should be scaled by the image SNR $%
\sigma _{T,k}(T(y_{k})|x,\ell )$. Note that the value of $\psi
_{D(y_{k})}^{(\theta (\ell ))}(x,\ell )$ should remain unchanged for $\theta
(\ell )>0$. Formally, this can be achieved by defining an extension of (\ref%
{eq:PropConj5}) as follows 
\begin{equation}
\varphi _{y_{k}}^{(j)}(x,\ell )\triangleq \psi _{D(y_{k})}^{(j)}(x,\ell ) 
\left[ \sigma _{T,k}(T(y_{k})|x,\ell )\right] ^{\delta _{0}[j]},
\label{eq: likelihoodratio}
\end{equation}%
In other words, for $j=0$, $\varphi _{y_{k}}^{(j)}(x,\ell )$ is equal to the
image SNR (\ref{eq:SNRV}) scaled by the mis-detection probability, otherwise
it is equal to the detection SNR (\ref{eq:SNRD}) scaled by the detection
probability.\newline
\indent Given a state $(x,\ell )$, $\varphi _{y_{k_{\!}}}^{(\theta (\ell
))}(x,\ell ) $ plays the same role as $\psi _{Z_{k_{\!}}}^{(\theta (\ell
))}(x,\ell )$, but accommodates both detection measurements and image
measurements. Moreover, since each object generates observation
independently from each other, the hybrid multi-object likelihood function
has the same form as (\ref{eq:RFSmeaslikelihood0}), but with $\psi
_{D(y_{k})}^{(\theta (\ell ))}(x,\ell )$ replaced by $\varphi
_{y_{k_{\!}}}^{(\theta (\ell ))}(x,\ell )$, i.e. 
\begin{equation}
g_{k}(y_{k}|\mathbf{X}_{k})\propto \sum_{\theta \in \Theta _{k}(\mathcal{L(}%
\mathbf{X}_{k}))}\prod\limits_{(x,\ell )\in \mathbf{X}_{k}}\varphi
_{y_{k_{\!}}}^{(\theta (\ell ))}(x,\ell ).  \label{eq:hybridlikelihood}
\end{equation}

In visual occlusions, the hybrid likelihood allows occluded objects to share
the image observations of the objects that occlude them. Moreover, when
integrated into the Bayes recursion (\ref{Chap_Kolmo_LRFS})-(\ref%
{Bayesupdate_LRFS}), consideration for a track-length-dependent survival
probability in combination with information from the image observation,
reduces uncertainties in the states of occluded objects and maintains their
existence probabilities to keep the tracks alive. Hence, hi-jacking and
premature track termination in longer occlusions will be avoided.

Remark: The hybrid multi-object likelihood function (\ref%
{eq:hybridlikelihood}) is a generalization of both the standard multi-object
likelihood and the separable likelihood in \cite{Vo2010}. When $%
P_{\!D,k}(x,\ell )=1$ for each $(x,\ell )\in \mathbf{X}_{k}$, i.e. there is
no mis-detection, the hybrid likelihood (\ref{eq:hybridlikelihood}) is the
same as the standard likelihood (\ref{eq:RFSmeaslikelihood0}). On the other
hand, if $P_{\!D,k}(x,\ell )=0$ for each $(x,\ell )\in \mathbf{X}_{k}$, i.e.
there is no detection, then the only non-zero term in the hybrid likelihood (%
\ref{eq:hybridlikelihood}) is one with $\theta (\ell )=0$ for all $\ell \in 
\mathcal{L(}\mathbf{X}_{k})$. In this case, the hybrid likelihood (\ref%
{eq:hybridlikelihood}) reduces to the separable likelihood in \cite{Vo2010}.
For a general detection profile $P_{\!D,k}$, the hybrid likelihood (\ref%
{eq:hybridlikelihood}) reduces to the standard likelihood (\ref%
{eq:RFSmeaslikelihood0}) when $\sigma _{T,k}(T(y_{k})|x,\ell )=1$ for each $%
(x,\ell )\in \mathbf{X}_{k}$.\newline
\indent Note that a hybrid likelihood function can be also developed for the
merged-measurement model. However, the resulting multi-object filter still
suffers from the same intractability as the merged-measurement filter.

\begin{figure}[tbp]
	\begin{center}
		\includegraphics[height=.45\linewidth]{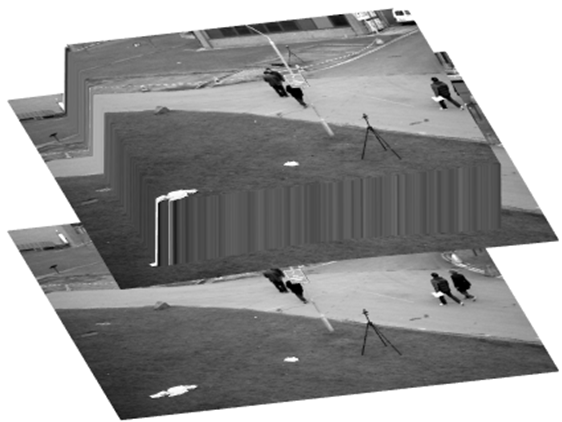} 
	\end{center}
	\caption{Example of scene mask for the proposed probability of survival}
	\label{Fig: mask}
\end{figure}

\subsection{Death model\label{subsec:Death}}

In most video MOT applications, if an object stays in the scene for a long
time, then it is more likely to continue to do so, provided it is not close
to the designated exit regions. Such prior empirical knowledge can be used
to improve occlusion handling, especially long occlusions that can lead to
premature track termination in on-line MOT algorithms. In general, the GLMB
filter would terminate an object that has not been detected over several
frames. However, if this object has been in the scene for some time and is
not in the proximity of designated exit regions, then it is highly likely to
be occluded and track termination should be delayed. The labeled RFS
formulation enables such prior information to be incorporated into track
termination in a principled manner, via the survival probability.

The labeled RFS formulation accommodates survival probabilities that depend
on track lengths since a labeled state contains the time of birth in its
label, and the track length is simply the difference between the current
time and the time of birth. In practice, it is unlikely for an object to
disappear in the middle of the visual scene (even if mis-detected or
occluded) whereas it is more likely to disappear near designated exit
regions due to the scene structure (e.g. the borders of the scene). Hence,
we require the survival probability to be large (close to one) in the middle
of the scene and small (close to zero) on the edges or designated death
regions. Furthermore, since objects staying in the scene for a long time are
more certain to continue existing, we require the survival probability to
increase to one as its track length increases.

An explicit form of the survival probability that satisfies these
requirements is given by 
\begin{equation}
P_{S,k}(x,\ell )=\frac{b(x)}{1+\text{exp}(-\gamma (k-\ell \lbrack 1,0]^{%
\text{T}}))}  \label{surv_prob}
\end{equation}%
where $b(x)$ is a scene mask that represents the scene structure, e.g.,
entrance or exit as illustrated in Fig. \ref{Fig: mask}, $\gamma $ is a
control parameter of the sigmoid function. The scene mask $b(x)$ can be
learned from a set of training data or designed from the known scene
structure.

\subsection{GLMB Recursion\label{subsec:VGLMB}}

A GLMB density can be written in the following form 
\begin{equation}
\mathbf{\pi }(\mathbf{X})=\Delta (\mathbf{X})\sum_{\xi \in \Xi
}\sum_{I\subseteq \mathbb{L}}\omega ^{(I,\xi )}\delta _{I}[\mathcal{L(}%
\mathbf{X})]\left[ p^{(\xi )}\right] ^{\mathbf{X}},  \label{eq:delta-glmb}
\end{equation}
where each $\xi \in \Xi $ represents a history of association maps $\xi
=(\theta _{1:k})$, each $p^{(\xi )}(\cdot ,\ell )$ is a probability density
on $\mathbb{X}$, and each $\omega ^{(I,\xi )}$ is non-negative with $%
\sum_{\xi \in \Xi }\sum_{I\subseteq \mathbb{L}}\omega ^{(I,\xi )}=1$. The
cardinality distribution of a GLMB is given by%
\begin{equation}
\Pr (\left\vert \mathbf{X}\right\vert =n)=\sum_{\xi \in \Xi
}\sum_{I\subseteq \mathbb{L}}\delta _{n}\left[ \left\vert I\right\vert %
\right] \omega ^{(I,\xi )},  \label{eq:GLMBCard}
\end{equation}%
while, the existence probability and probability density of track $\ell \in 
\mathbb{L}$ are respectively\allowdisplaybreaks%
\begin{align}
r(\ell )& =\sum_{\xi \in \Xi }\sum_{I\subseteq \mathbb{L}}1_{I}(\ell )\omega
^{(I,\xi )}, \\
p(x,\ell )& =\frac{1}{r(\ell )}\sum_{\xi \in \Xi }\sum_{I\subseteq \mathbb{L}%
}1_{I}(\ell )\omega ^{(I,\xi )}p^{(\xi )}(x,\ell ).
\end{align}

Given the GLMB density (\ref{eq:delta-glmb}), an intuitive multi-object
estimator is the \emph{multi-Bernoulli estimator}, which first determines
the set of labels $L\subseteq $ $\mathbb{L}$ with existence probabilities
above a prescribed threshold, and second the MAP/mean estimates from the
densities $p(\cdot ,\ell ), \ell \in L$, for the states of the objects. A
popular estimator is a suboptimal version of the Marginal Multi-object
Estimator \cite{Mahler_book07}, which first determines the pair $(L,\xi )$
with the highest weight $\omega ^{(L,\xi )}$ such that $\left\vert
L\right\vert $ coincides with the MAP cardinality estimate, and second the
MAP/mean estimates from $p^{(\xi )}(\cdot ,\ell ),\ell \in L$, for the
states of the objects.

The GLMB family enjoys a number of nice analytical properties. The void
probability functional--a necessary and sufficient statistic--of a GLMB, the
Cauchy-Schwarz divergence between two GLMBs, the $L_{1}$-distance between a
GLMB and its truncation, can all be computed in closed form \cite{VVP14}.
The GLMB is flexible enough to approximate any labeled RFS density with
matching intensity function and cardinality distribution \cite{Papi_etal15}.
More importantly, the GLMB family is closed under the prediction equation (%
\ref{Chap_Kolmo_LRFS}) and conjugate with respect to the standard
observation likelihood \cite{VV13}.\newline
\indent In the following we show that the GLMB family is conjugate with
respect to the hybrid observation likelihood function. Hence, starting from
an initial GLMB prior, all multi-object predicted and updated densities
propagated by the Bayes recursion (\ref{Chap_Kolmo_LRFS})-(\ref%
{Bayesupdate_LRFS}) are GLMBs. For notational compactness, we drop the
subscript $k$ for the current time, the next time is indicated by the
subscript `$+$'.

\begin{proposition}
\label{PropBayes_strong}Suppose that the multi-object prediction density to
time $k+1$ is a\ GLMB of the form 
\begin{equation}
\mathbf{\bar{\pi}}_{_{\!}+}(_{_{\!}}\mathbf{X}_{+})=\Delta _{{\!}}(\mathbf{X}%
_{+_{\!}})\sum_{\xi ,I_{+}}\bar{\omega}_{+}^{(\xi ,I_{+\!})}\delta
_{_{\!}I_{+\!}}[\mathcal{L}(\mathbf{X}_{+})]\!\left[ \bar{p}_{+}^{(\xi )}%
\right] ^{\mathbf{X}_{+}},  \label{eq:GLMB_pred0}
\end{equation}%
where $\xi \in \Xi $, $I_{+}\in \mathcal{F}(\mathbb{L}_{+})$. Then under the
hybrid observation likelihood function (\ref{eq:hybridlikelihood}), the
filtering density at time $k+1$ is a GLMB of the form%
\begin{equation}
\mathbf{\pi }_{\!y_{_{\!}+}\!}(\mathbf{X}_{+})\propto \Delta _{\!}(\mathbf{X}%
_{+\!})\sum_{\xi ,I_{+},\theta _{+}}\omega _{y_{_{\!}+}}^{(_{\!}\xi
,I_{+\!},\theta _{\!+\!})}\delta _{I_{+\!\!}}[\mathcal{L}(\mathbf{X}_{+})]\!%
\left[ p_{+}^{(\xi _{\!},\theta _{\!+})\!}\right] ^{\mathbf{X}_{+}}\!,
\label{eq:GLMB_upd0}
\end{equation}%
where $\theta _{+}\in \Theta _{+}$, and \allowdisplaybreaks%
\begin{eqnarray}
\!\!\!\!\!\!\!\!\omega _{y_{_{\!}+}}^{(_{\!}\xi ,I_{+\!},\theta
_{\!+\!})}\!\!\! &=&\!\!\!\bar{\omega}_{+}^{(\xi ,I_{+\!})}1_{\Theta
_{+}(_{\!}I_{+\!})}(\theta _{+})\left[ \bar{\varphi}_{y_{+}}^{(\xi ,\theta
_{_{\!}+})}\right] ^{I_{+\!}},  \label{eq:GLMB_upd1} \\
\!\!\!\!\!\!\!\!\bar{\varphi}_{y_{+}}^{(\xi ,\theta _{_{\!}+})}(\ell
_{_{\!}+_{\!}})\!\!\! &=&\!\!\!\left\langle \bar{p}_{+}^{(\xi )}(\cdot ,\ell
_{+}),\varphi _{y_{+_{\!}}}^{(\theta _{+}(\ell _{+}))}(\cdot ,\ell
_{+})\right\rangle,  \label{eq:GLMB_upd2} \\
\!\!\!\!\!\!\!\!p_{_{+}}^{(\xi ,\theta _{_{\!}+_{\!}})_{\!}}(x_{_{\!}+},\ell
_{_{\!}+_{\!}})\!\!\! &=&\!\!\!\frac{\bar{p}_{+}^{(\xi )}(x_{+},\ell
_{+})\varphi _{y_{+_{\!}}}^{(\theta _{+}(\ell _{+}))}(x_{+},\ell _{+})}{\bar{%
\varphi}_{y_{+}}^{(\xi ,\theta _{_{\!}+})}(\ell _{+})}.  \label{eq:GLMB_upd3}
\end{eqnarray}
\end{proposition}

\begin{proof}
Note that the likelihood function (\ref{eq:hybridlikelihood}) at time $k+1$
can be written as%
\begin{equation*}
g_{+}(y_{+}|\mathbf{X}_{+})\propto \sum_{\theta _{+}}1_{\Theta _{+}(\mathcal{%
L(}\mathbf{X}_{+}))}(\theta _{+})\left[ \tilde{\varphi}_{y_{+_{\!}}}^{(%
\theta _{+})}\right] ^{\mathbf{X}_{+}},
\end{equation*}%
where $\tilde{\varphi}_{y_{+_{\!}}}^{(\theta _{+})}(x,\ell )\triangleq
\varphi _{y_{+_{\!}}}^{(\theta _{+}(\ell ))}(x,\ell )$. Using Bayes rule%
\begin{eqnarray*}
\mathbf{\pi }_{\!y_{_{\!}+}\!}(\mathbf{X}_{+}) &=&\mathbf{\bar{\pi}}%
_{_{\!}+}(_{_{\!}}\mathbf{X}_{+})g_{+}(y_{+}|\mathbf{X}_{+}) \\
&\propto &\Delta _{\!}(\mathbf{X}_{+})\mathbf{\!}\sum_{\xi ,I_{+}}\!\bar{%
\omega}_{+}^{(\xi ,I_{+\!})}\delta _{_{\!}I_{+\!}}[\mathcal{L}(\mathbf{X}%
_{+})]\!\left[ \bar{p}_{+}^{(\xi )}\right] ^{\mathbf{X}_{+}} \\
&&\times \sum_{\theta _{+}}1_{\Theta _{+}(\mathcal{L(}\mathbf{X}%
_{+}))}(\theta _{+})\left[ \tilde{\varphi}_{y_{+_{\!}}}^{(\theta _{+})}%
\right] ^{\mathbf{X}_{+}} \\
&=&\Delta _{\!}(\mathbf{X}_{+})\mathbf{\!}\sum_{\xi ,I_{+},\theta _{+}}\bar{%
\omega}_{+}^{(\xi ,I_{+\!})}1_{\Theta _{+}(_{\!}I_{+\!})}(\theta _{+}) \\
&&\times \delta _{_{\!}I_{+\!}}[\mathcal{L}(\mathbf{X}_{+})]\!\left[ \bar{p}%
_{+}^{(\xi )}\tilde{\varphi}_{y_{+_{\!}}}^{(\theta _{+})}\right] ^{\mathbf{X}%
_{+}} \\
&=&\Delta _{\!}(\mathbf{X}_{+})\mathbf{\!}\sum_{\xi ,I_{+},\theta _{+}}\!%
\bar{\omega}_{+}^{(\xi ,I_{+\!})}1_{\Theta _{+}(_{\!}I_{+\!})}(\theta _{+})
\\
&&\times \delta _{_{\!}I_{+\!}}[\mathcal{L}(\mathbf{X}_{+})]\!\left[ \bar{%
\varphi}_{y_{+}}^{(\xi ,\theta _{_{\!}+})}\right] ^{\mathcal{L}(\mathbf{X}%
_{+})}\!\!\left[ \frac{\bar{p}_{+}^{(\xi )}\tilde{\varphi}%
_{y_{+_{\!}}}^{(\theta _{+})}}{\bar{\varphi}_{y_{+}}^{(\xi ,\theta
_{_{\!}+})}}\right] ^{\mathbf{X}_{+}} \\
&=&\Delta _{\!}(\mathbf{X}_{+})\mathbf{\!}\sum_{\xi ,I_{+},\theta _{+}}\!%
\bar{\omega}_{+}^{(\xi ,I_{+\!})}1_{\Theta _{+}(_{\!}I_{+\!})}(\theta _{+})%
\left[ \bar{\varphi}_{y_{+}}^{(\xi ,\theta _{_{\!}+})}\right] ^{I_{+\!}} \\
&&\times \delta _{I_{+\!\!}}[\mathcal{L}(\mathbf{X}_{+})]\!\left[
p_{+}^{(\xi _{\!},\theta _{\!+})\!}\right] ^{\mathbf{X}_{+}}.
\end{eqnarray*}
\end{proof}

In this work we adopt the joint prediction and update strategy \cite{VVH}
for the proposed video MOT GLMB filter. Using the same line of arguments as
in \cite{VVH}, yields the following recursion

\begin{proposition}
Given the GLMB filtering density (\ref{eq:delta-glmb}) at time $k$, the
filtering density at time $k+1$ is: \allowdisplaybreaks%
\begin{equation}
\!\!\mathbf{\pi }_{\!+\!}(\mathbf{X})\!\propto \Delta _{\!}(_{\!}\mathbf{X}%
_{\!})\!\!\!\sum\limits_{I_{\!},\xi ,I_{\!+\!},\theta _{\!+\!\!}}\!\!\omega
^{(I,\xi )}\omega _{y_{_{\!}+}}^{(_{\!}I_{\!},\xi ,I_{\!+\!},\theta
_{\!+\!})}\delta _{_{\!}I_{+}}[\mathcal{L}(_{\!}\mathbf{X}_{\!})]\!\left[
p_{y_{_{\!}+_{\!}}}^{(_{\!}\xi ,\theta _{\!+\!})}{}_{\!}\right] ^{\!\mathbf{X%
}}\!,  \label{eq:GLMB_joint0}
\end{equation}%
where $I\in \mathcal{F}(\mathbb{L})$, $\xi \in \Xi $, $I_{+}\in \mathcal{F}(%
\mathbb{L}_{+})$, $\theta _{+}\in {\Theta }_{+}(I_{+})$, and%
\allowdisplaybreaks%
\begin{eqnarray}
\!\!\!\!\!\!\omega _{y_{_{\!}+}}^{(_{\!}I_{\!},\xi ,I_{\!+\!},\theta
_{\!+\!})}\!\!\! &=&\!\!\!\left[ 1-\bar{P}_{S}^{(\xi )}\right]
^{\!I-I_{\!+}}\!\left[ \bar{P}_{S\!}^{(\xi )}\right] ^{\!I\cap I_{+\!}} 
\notag \\
\!\!\!\! &\times &\!\!\!\left[ 1-r_{B\!,+}\right] ^{\mathbb{B}%
_{\!+}-I_{\!+}}r_{B\!,+}^{\mathbb{B}_{\!{+}}\cap I_{+}}\!\left[ \bar{\varphi}%
_{y_{+}}^{(\xi ,\theta _{_{\!}+\!})}\right] ^{I_{+}},  \label{eq:GLMB_joint1}
\\
\!\!\!\!\!\!\bar{P}_{S}^{(\xi )}(\ell )\!\!\! &=&\!\!\!\left\langle p^{(\xi
)}(\cdot ,\ell ),P_{S}(\cdot ,\ell )\right\rangle,  \label{eq:GLMB_joint2} \\
\!\!\!\!\!\!\bar{\varphi}_{y_{+}}^{(\xi ,\theta _{_{\!}+\!})}(\ell
_{_{\!}+})\!\!\! &=&\!\!\!\left\langle \bar{p}_{+}^{(\xi )}(\cdot ,\ell
_{_{\!}+}),\varphi _{y_{+}}^{(\theta _{_{\!}+}(\ell _{_{\!}+}))}(\cdot ,\ell
_{_{\!}+})\right\rangle,  \label{eq:GLMB_joint3} \\
\!\!\!\!\!\!\bar{p}_{+}^{(\xi )_{\!}}(x_{_{\!}+},\ell _{_{\!}+})\!\!\!
&=&\!\!\!1_{\mathbb{L}_{\!}}(\ell _{_{\!}+\!})\frac{\!\left\langle
P_{S}(\cdot ,\ell _{_{\!}+\!})f_{_{\!}+\!}(x_{_{\!}+}|\cdot ,\ell
_{_{\!}+\!}),p^{(\xi )}(\cdot ,\ell _{_{\!}+\!})\right\rangle }{\bar{P}%
_{S}^{(\xi )}(\ell _{_{\!}+})}  \notag \\
\!\!\!\! &+&\!\!\!1_{\mathbb{B}_{+}}\!(\ell
_{_{\!}+_{\!}})p_{B,+_{_{\!}}}(x_{_{\!}+},\ell _{_{\!}+}),
\label{eq:GLMB_joint4} \\
\!\!\!\!\!\!p_{+}^{(\xi _{\!},\theta _{\!+})\!}(x_{_{\!}+},\ell
_{_{\!}+})\!\!\! &=&\!\!\!\frac{\bar{p}_{+}^{(\xi )}(x_{_{\!}+},\ell
_{_{\!}+})\varphi _{y_{+}}^{(\theta _{_{\!}+}(\ell
_{_{\!}+}))}(x_{_{\!}+},\ell _{_{\!}+})}{\bar{\varphi}_{y_{+}}^{(\xi ,\theta
_{_{\!}+})}(\ell _{_{\!}+})}.  \label{eq:GLMB_joint5}
\end{eqnarray}
\end{proposition}

The summation in (\ref{eq:GLMB_joint0}) can be interpreted as an enumeration
of all possible combinations of births, deaths and survivals together with
associations of new measurements to hypothesized tracks. Observe that (\ref%
{eq:GLMB_joint0}) does indeed take on the same form as (\ref{eq:delta-glmb}%
) when rewritten as a sum over $I_{+},\xi ,\theta _{+}$ with weights 
\begin{equation}
\omega _{+}^{(I_{+},\xi ,\theta _{+})}\propto \sum\limits_{I}\omega ^{(I,\xi
)}\omega _{y_{+}}^{(I,\xi ,I_{+},\theta _{+})}.  \label{eq:GLMB_joint6}
\end{equation}%
Hence at the next iteration we only propagate forward the components $%
(I_{+},\xi ,\theta _{+})$ with weights $\omega _{+}^{(I_{+},\xi ,\theta
_{+})}$.\newline
Remark: It is also possible to approximate the resulting GLMB filtering
density by an LMB with matching 1st moment and cardinality distribution \cite%
{Reuter14}. This so-called LMB filtering strategy reduces considerable
computations since an LMB is a GLMB with one term. However, tracking
performance tend to degrade, especially in scenarios with many closely space
objects.

Note that for high SNR scenarios the detection probability is high, hence
the recursion (\ref{eq:GLMB_joint0})-(\ref{eq:GLMB_joint6}) would process
detections mostly. On the other hand when the detection probability is low
it would process the image mostly. In practice the SNR varies between
different regions in the observation space as well as with time, the
recursion (\ref{eq:GLMB_joint0})-(\ref{eq:GLMB_joint6}) adaptively processes
detections and image data to improve performance while reducing the
computational cost.

\subsection{GLMB Filter Implementation\label{subsec:Implementation}}

The number of terms in the GLMB filtering density grows super-exponentially,
and it is necessary to truncate these terms without exhaustive enumeration.
A two-stage implementation of the GLMB filter truncates the prediction and
filtering densities using the K-shortest path and the ranked assignment
algorithms, respectively \cite{VVP14}. In \cite{VVH} the joint prediction
and update was designed to improve the truncation efficiency of the
two-staged implementation. Further, the GLMB truncation can be performed via
Gibbs sampling with linear complexity in the number of detections (the
reader is referred to \cite{VVH} for derivations and analysis).
Fortuitously, this implementation can be readily adapted for the video MOT
GLMB filter (\ref{eq:GLMB_joint0})-(\ref{eq:GLMB_joint6}).

The GLMB filtering density (\ref{eq:delta-glmb}) at time $k$ is completely
characterized by the parameters $(\omega ^{(I,\xi )},p^{(\xi )})$, $(I,\xi
)\in \mathcal{F}\!(\mathbb{L})\!\times \!\Xi $, which can be enumerated as $%
\{(I^{(h)},\xi ^{(h)},\omega ^{(h)},p^{(h)})\}_{h=1}^{H}$, where%
\begin{equation*}
\omega ^{(h)}\triangleq \omega ^{(I^{(h)},\xi ^{(h)})},\;p^{(h)}\triangleq
p^{(\xi ^{(h)})}.
\end{equation*}%
Since (\ref{eq:delta-glmb}) can now be rewritten as 
\begin{equation*}
\mathbf{\pi }(\mathbf{X})=\Delta (\mathbf{X})\sum_{h=1}^{H}\omega
^{(h)}\delta _{I^{(h)}}[\mathcal{L(}\mathbf{X})]\left[ p^{(h)}\right] ^{%
\mathbf{X}},
\end{equation*}%
implementing the GLMB filter amounts to propagating the component set $%
\{(I^{(h)},\omega ^{(h)},p^{(h)})\}_{h=1}^{H}$ (there is no need to store $%
\xi ^{(h)}$) forward in time using (\ref{eq:GLMB_joint0})-(\ref%
{eq:GLMB_joint6}). The procedure for computing the component set $%
\{(I_{+}^{(h_{+})},\omega
_{+}^{(h_{+})},p_{+}^{(h_{+})})\}_{h_{+}=1}^{H_{+}} $ at the next time is
summarized in Algorithm 1. Note that to be consistent with the indexing by $%
h $ instead of $(I,\xi )$, we also abbreviate%
\begin{eqnarray}
\!\!\!\!\!\bar{P}_{S}^{(h)}(\ell _{i})\!\! &\triangleq &\!\!\bar{P}%
_{S}^{(\xi ^{(h)})\!}(\ell _{i}),\;\;\bar{p}_{+}^{(h)}(x,\ell
_{i})\triangleq \bar{p}_{+}^{(\xi ^{(h)})_{\!}}(x,\ell _{i}),  \notag \\
\!\!\!\!\!\bar{\varphi}_{y_{+}}^{(h,j)}(\ell _{i})\!\! &\triangleq
&\!\!\left\langle \bar{p}_{+}^{(h)_{\!}}(\cdot ,\ell _{i}),\varphi
_{y_{+}}^{(j)}(\cdot ,\ell _{i})\right\rangle ,  \notag \\
\!\!\!\!\!\eta _{i}^{(h)}(j)\!\! &\triangleq &\!\!%
\begin{cases}
1-\bar{P}_{S}^{(h)}(\ell _{i}), & \!\ell _{i\!}\in I^{(h)},\text{ }j\!<0\!,
\\ 
\bar{P}_{S}^{(h)}(\ell _{i})\bar{\varphi}_{y_{+}}^{(h,j)}(\ell _{i}), & 
\!\ell _{i\!}\in I^{(h)},\text{ }j\!\geq 0, \\ 
1-r_{B\!,+}(\ell _{i}), & \!\ell _{i\!}\in \mathbb{B}_{+},\text{ }j\!<0, \\ 
r_{B\!,+}(\ell _{i})\bar{\varphi}_{y_{+}}^{(h,j)}(\ell _{i}), & \!\ell
_{i\!}\in \mathbb{B}_{+},\text{ }j\!\geq 0.%
\end{cases}
\label{eq:eta}
\end{eqnarray}

\hrule

\vspace{5pt}

\textbf{Algorithm 1. Joint Prediction and Update}

\begin{itemize}
\item {\footnotesize \textsf{input: }}$\{(I^{(h)},\omega
^{(h)},p^{(h)})\}_{h=1}^{H}$, $y_{+}$, $Z_{+}$, $H_{+}^{\max }$,

\item {\footnotesize \textsf{input: }}$\{(r_{\!B\!,+}^{(\ell
)},p_{B\!,+}^{(\ell )})\}_{\ell \in \mathbb{B}_{\!+}}$, $P_{S\!}$, $%
f_{\!+\!} $, $\sigma_{\!D,+\!}$, $\sigma _{\!T,+\!}$,

\item \textsf{{\footnotesize {output: }}}$\{(I_{+}^{(h_{+})},\omega
_{+}^{(h_{+})},p_{+}^{(h_{+})})\}_{h_{+}=1}^{H_{+}}$
\end{itemize}

\vspace{2pt}

\hrule

\vspace{2pt}

\textsf{{\footnotesize {sample counts }}}$[T_{+}^{(h)}]_{h=1}^{H}$\textsf{\ 
{\footnotesize {\ \ from\textsf{\ multinomial} distribution}}}

\textsf{{\footnotesize {with parameters }}}$H_{+}^{\max }$\textsf{\ 
{\footnotesize {\ trials and weights }}}${[}\omega ^{(h)}]_{h=1}^{H}$

\textsf{{\footnotesize {for}}}\textsf{\ }$h=1:H$

\quad \textsf{{\footnotesize {compute }}}$\eta ^{(h)}:=[\eta
_{i}^{(h)}(j)]_{(i,j)=(1,-1)}^{(|I^{(h)}\!\cup \!\mathbb{B}_{+}|,|Z_{+}|)}$
\ \textsf{{\footnotesize {using }}}(\ref{eq:eta})

\quad \textsf{{\footnotesize {initialize }}}$\gamma ^{(h,1)}$

\quad $\{\gamma ^{(h,t)}\}_{_{t=1}}^{\tilde{T}_{+}^{(h)}}:=\mathsf{\
Unique(Gibbs}(\gamma ^{(h,1)},T_{+}^{(h)},\eta ^{(h)}));$

\quad \textsf{{\footnotesize {for }}}$t=1:\tilde{T}_{+}^{(h)}$

\quad \quad \textsf{{\footnotesize {compute }}}$I_{+}^{(h,t)}$\textsf{\ 
{\footnotesize {\ from }}}$I^{(h)}$\textsf{{\footnotesize {\ and }}}$\gamma
^{(h,t)}$\textsf{{\footnotesize {\ using }}}(\ref{eq:Iplus})

\quad \quad \textsf{{\footnotesize {compute }}}$\omega _{+}^{(h,t)}$\textsf{%
\ {\footnotesize {\ from }}}$\omega ^{(h)}$\textsf{{\footnotesize {\ and }}}$%
\gamma ^{(h,t)}$\textsf{{\footnotesize {\ using }}}(\ref{eq:wplus})

\quad \quad \textsf{{\footnotesize {compute }}}$p_{+}^{(h,t)}$\ \textsf{\ 
{\footnotesize {from }}}$p^{(h)}$\textsf{{\footnotesize {\ and }}}$\gamma
^{(h,t)}$\textsf{{\footnotesize {\ using}}} (\ref{eq:pplus})

\quad \textsf{{\footnotesize {end}}}

\textsf{{\footnotesize {end}}}

$(\{(I_{+}^{(h_{+})}\!,p_{+}^{(h_{+})})\}_{h_{+}=1}^{H_{+}},\sim ,[U_{h,t}])$

\quad \quad \quad \quad \quad \quad \quad \quad $:=\mathsf{Unique}%
(\{(I_{+}^{(h,t)}\!,p_{+}^{(h,t)})\}_{(h,t)=(1,1)}^{(H,\tilde{T}
_{+}^{(h)})});$

\textsf{{\footnotesize {for }}}$h_{+}=1:H_{+}$

\quad $\omega _{+}^{(h_{+})}:=\sum\limits_{h,t:U_{h,t}=h_{+}}\omega
_{+}^{(h,t)};$

\textsf{{\footnotesize {end}}}

\textsf{{\footnotesize {normalize weights }}}$\{\omega
_{+}^{(h_{+})}\}_{h_{+}=1}^{H_{+}}$

\vspace{3pt}

\hrule

\vspace{2pt}

\vspace{5pt}

\hrule

\vspace{3pt}

\vspace{5pt}

\textbf{Algorithm 1a. Gibbs }

\begin{itemize}
\item \textsf{{\footnotesize {input: }}}$\gamma ^{(1)},T,\eta =[\eta
_{i}(j)] $

\item \textsf{{\footnotesize {output: }}}$\gamma ^{(1)},...,\gamma ^{(T)}$
\end{itemize}

\vspace{2pt}

\hrule

\vspace{7pt}

$P:=\mathsf{size}(\eta ,1);\quad M:=\mathsf{size}(\eta ,2)-2;\quad
c:=[-1:M]; $

\textsf{{\footnotesize {for}}\ }$t=2:T$

\quad $\gamma ^{(t)}:=[$ $];$

\quad \textsf{{\footnotesize {for }}}$n=1:P$

\quad \quad \textsf{{\footnotesize {for }}}$j=1:M$

\quad \quad \quad $\eta_{n}(j):=\eta_{n}(j)(1-1_{\{\gamma
_{1:n-1}^{(t)},\gamma _{n+1:P}^{(t-1)}\}}(j));$

\quad \quad \textsf{{\footnotesize {end}}}

\quad \quad $\gamma _{n}^{(t)}\sim \mathsf{Categorical}(c,\eta_{n});$\quad $%
\gamma ^{(t)}:=[\gamma ^{(t)},\gamma _{n}^{(t)}];$

\quad \textsf{{\footnotesize {end}}}

\textsf{{\footnotesize {end}}}

\vspace{5pt} \hrule
\vspace{7pt}

There are three main steps in one iteration of the GLMB filter. First, the
Gibbs sampler (Algorithm 1a) is used to generate the \emph{auxiliary vectors}
$\gamma ^{(h,t)}$, $h=1$:$H$, $t=1$:$\tilde{T}_{+}^{(h)}$, with the most
significant weights $\omega _{+}^{(h,t)}$(note that $\gamma ^{(h,t)}$ is an
equivalent representation of the hypothesis $(I_{+}^{(h,t)},\theta
_{+}^{(h,t)})$,\ with components $\gamma _{i}^{(h,t)}$, $i=1$:$%
|I^{(h)}\!\cup \!\mathbb{B}_{+}|$, defined as $\theta _{+}^{(h,t)}(\ell
_{i}) $ when $\ell _{i}\in I^{(h)}$, and $-1$ otherwise \cite{VVH}). The
Gibbs sampler has an exponential convergence rate \cite{VVH}. More
importantly, it is not necessary to discard burn-ins and wait for samples
from the stationary distribution. All distinct samples can be used, the
larger the weights, the smaller the $L_{1}$ error from the true GLMB
filtering density \cite{VVH}.\newline
\indent Second, the auxiliary vectors are used to generate an intermediate
set of parameters with the most significant weights $(I^{(h)},I_{+}^{(h,t)},%
\omega _{+}^{(h,t)},p_{+}^{(h,t)})$, $h=1$:$H$, $t=1$:$\tilde{T}_{+}^{(h)}$,
via (\ref{eq:GLMB_joint0}). Note that given a component $h$ and $\gamma
^{(h,t)}$, it can be shown that \cite{VVH} 
\begin{eqnarray}
I_{+}^{(h,t)} &=&\{\ell _{i}\in I^{(h)}\cup \mathbb{B}_{\!{+}\!}:\gamma
_{i}^{(h,t)}\geq 0\},  \label{eq:Iplus} \\
\omega _{+}^{(h,t)} &\propto &\omega ^{(h)}\prod\limits_{i=1}^{|I^{(h)}\cup 
\mathbb{B}_{+}|}\eta _{i}^{(h)}(\gamma _{i}^{(h,t)}),  \label{eq:wplus} \\
p_{+}^{(h,t)\!}(\cdot ,\ell _{i}) &=&\frac{\bar{p}_{+}^{(h)}(\cdot ,\ell
_{i})\varphi _{y_{+}}^{(\gamma _{i}^{(h,t)})}(\cdot ,\ell _{i})}{\bar{\varphi%
}_{y_{+}}^{(h,\gamma _{i}^{(h,t)})}(\ell _{i})}.  \label{eq:pplus}
\end{eqnarray}%
Note also that $\theta _{+}^{(h,t)}(\ell _{i})=\gamma _{i}^{(h,t)}$ when $%
\gamma_{i}^{(h,t)}\geq 0$, for $\ell _{i}\in I_{+}^{(h,t)}$.

Third, the intermediate parameters are marginalized via (\ref{eq:GLMB_joint6}%
) to give the new parameter set $\{(I_{+}^{(h_{+})},\omega
_{+}^{(h_{+})},p_{+}^{(h_{+})})\}_{h_{+}=1}^{H_{+}}$. Note that $U_{h,t}$
gives the index of the GLMB component at time $k+1$ that $%
(I^{(h)},I_{+}^{(h,t)},p_{+}^{(h,t)})$ contributes to.\newline

\section{Experimental results\label{sec:Experiment}}

The proposed MOT filter is tested on a simulated TBD\ application in
subsection \ref{subsec:TBD}, and on real video data in subsection \ref%
{subsec:visual_tracking}.

\subsection{TBD\label{subsec:TBD}}

\subsubsection{Dynamic motion and observation model}

Consider a scenario with upto 5 objects, each with a 4D state $%
x_{k}=[~p_{x,k},\dot{p}_{x,k},p_{y,k},\dot{p}_{y,k}~]^{\text{T}}$ of
position and velocity. Each object follows a constant velocity model with
Gaussian transition density 
\begin{equation*}
f_{k|k-1}(x_{k}|x_{k-1})=\mathcal{N}(x_{k};Fx_{k-1},Q),
\end{equation*}%
where 
\begin{equation*}
F=I_{2}\otimes 
\begin{bmatrix}
1 & T_{s} \\[0.3em]
0 & 1%
\end{bmatrix}%
,
\end{equation*}%
$I_{2}$ is the 2 $\times $ 2 identity matrix, $\otimes $ denotes the
Kronecker product, $T_{s}$ is the sampling period of the video data, $%
Q=\sigma _{v}^{2}I_{2}$, and $\sigma _{v}=1$ pixels/frame is the noise
standard deviation.\newline
\indent The birth density is assumed to be LMB with 5 components of 0.03
birth probability and Gausssian distributed birth densities as 
\begin{equation*}
\begin{array}{lll}
\mathcal{N}(\cdot ,[5;0;5;0]^{T},P_{\gamma }),~~\mathcal{N}(\cdot
,[5;0;25;0]^{T},P_{\gamma }), &  &  \\ 
\mathcal{N}(\cdot ,[5;0;90;0]^{T},P_{\gamma }),~~\mathcal{N}(\cdot
,[90;0;30;0]^{T},P_{\gamma }), &  &  \\ 
\mathcal{N}(\cdot ,[80;0;90;0]^{T},P_{\gamma }),~~P_{\gamma
}=diag([3;2;3;2]). &  & 
\end{array}%
\end{equation*}%
The survival probability $P_{S}$ for the standard GLMB filter is 0.98 and
the control parameter $\gamma $ of the age-dependent survival probability is
set to $0.1$. The scene mask $b(x)$ of the same shape as Fig. \ref{Fig: mask}
with a margin of 10 pixels around the border area is used. \newline
\indent The observations are raw images simulated from the radar TBD
measurement model \cite{PapiKim15}, consisting of an array of pixel values
representing the power signal returns i.e., $y_{k}=[y^{(1)},...,y^{(i)}]$,
with%
\begin{equation}
y^{(i)}=\left\vert \displaystyle\sum\limits_{\mathbf{x}\in \mathbf{X}:i\in 
\mathcal{C}(\mathbf{x})}A(\mathbf{x})h_{A}^{(i)}(\mathbf{x}%
)+w^{(i)}\right\vert ^{2},  \label{eq:pixel_return}
\end{equation}%
where $\mathcal{C}(\mathbf{x})$ is usually referred to as the target
template, $A(\mathbf{x})$ denotes the amplitude of the return signal. 
\begin{equation}
h_{A}^{(i)}(\mathbf{x})=\text{exp}\left( -\frac{(r_{i}-r(\mathbf{x}))^{2}}{2R%
}-\frac{(s_{i}-s(\mathbf{x}))^{2}}{2S}\right) ,  \label{eq:point_spread}
\end{equation}%
is the point spread function value in cell $i$ from state $\mathbf{x}$, $R=1$
and $S=1$ are constants related to the image cell resolution; $r(\mathbf{x})$
and $s(\mathbf{x})$ are the coordinates of the object in the measurement
space; $r_{i}$ and $s_{i}$ are the cell centroids. \newline
Remark: Setting a relatively high SNR for the simulation means that the filter will mostly operate like a standard GLMB filter, while a low SNR means it mostly operates like a TBD-GLMB filter. Neither scenarios are interesting. In this example we simulate the observations with SNRs that fluctuate between 10dB and 7dB within the same image. Further, to demonstrate how the tracker adapts to the SNR mismatch, the observation model used by the tracker is instantiated with a 10dB SNR.

The detection and transformed image observation for the raw pixel image model (\ref%
{eq:pixel_return})-(\ref{eq:point_spread}) are obtained as follows.\newline
\indent A hard thresholding is applied to the raw image $y_{k}$, and the
detection model use by the proposed filter consists of a single-object
detection likelihood $g_{D,k}(z|x,\ell )=\mathcal{N}(z;Hx,\Sigma )$, where $%
H=[1~0~0~0;~0~0~1~0]$; $\Sigma =diag(4^{2},4^{2})$, a detection probability $%
P_{D}$ of 0.98, and a clutter rate of 10 points per frame. On the other
hand, the transformed image $T(y_{k})$ is the correlation response between
the reference template and the observed template, obtained from the raw
image $y_{k}$ via Kernelized Correlation Filtering (KCF) \cite{Henriques15}.
The image observation model use by the proposed filter is given by 
\begin{equation*}
g_{T,k}(T(y_{k})|x,\ell )\propto \text{exp}\left( -\frac{1}{\sigma ^{2}}%
\left( \left\Vert \text{f}(x)-\bar{\text{f}}_{\ell }\right\Vert ^{2}\right)
\right) ,
\end{equation*}%
where $\text{f}(x)$ is the observed template at a given object state $x$; $%
\bar{\text{f}}_{\ell }$ is the reference template of the track label $\ell $
(consists of pixel intensities in a 3 pixel by 3 pixel region); $\sigma $
controls the shape of the function. The Unscented Transform is used for the
measurement update for image observations. To adapt to appearance changes,
pixels inside regions with confident point detections are used to update $%
\bar{\text{f}}_{\ell }$. Empirically, this strategy is more robust than the
update scheme in \cite{Henriques15} because accumulated learning errors are
reduced by updating the model with confident detections. \newline
\begin{figure}[tbp]
	\begin{center}
		\includegraphics[height=.60\linewidth]{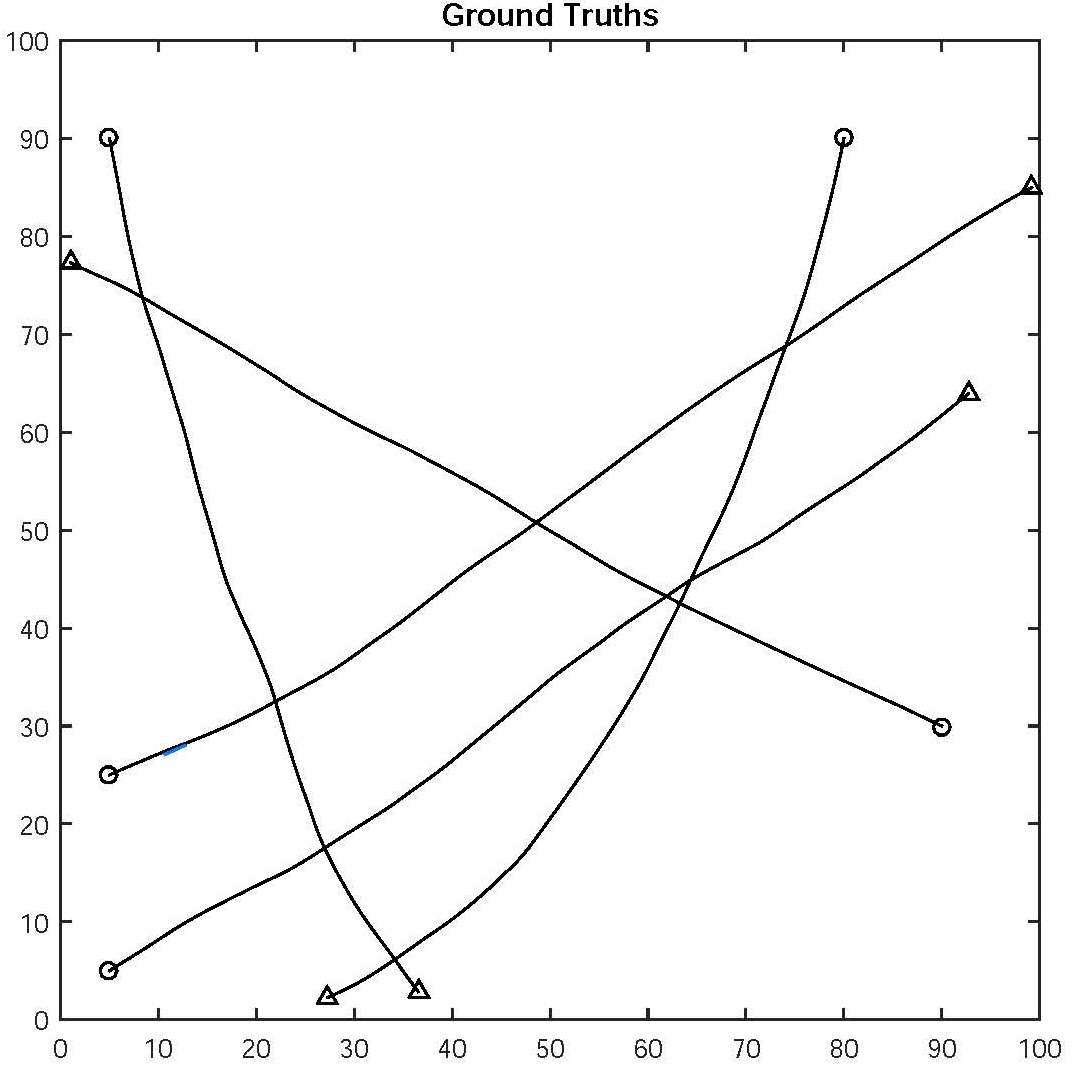} 
	\end{center}
	\caption{True tracks in the x y plane. Start/Stop positions are shown with $\circ/ \triangle$.}
\label{Fig: groundtruth} 
\end{figure}
\begin{figure}[tbp]
	\begin{center}
		\includegraphics[height=1.2\linewidth]{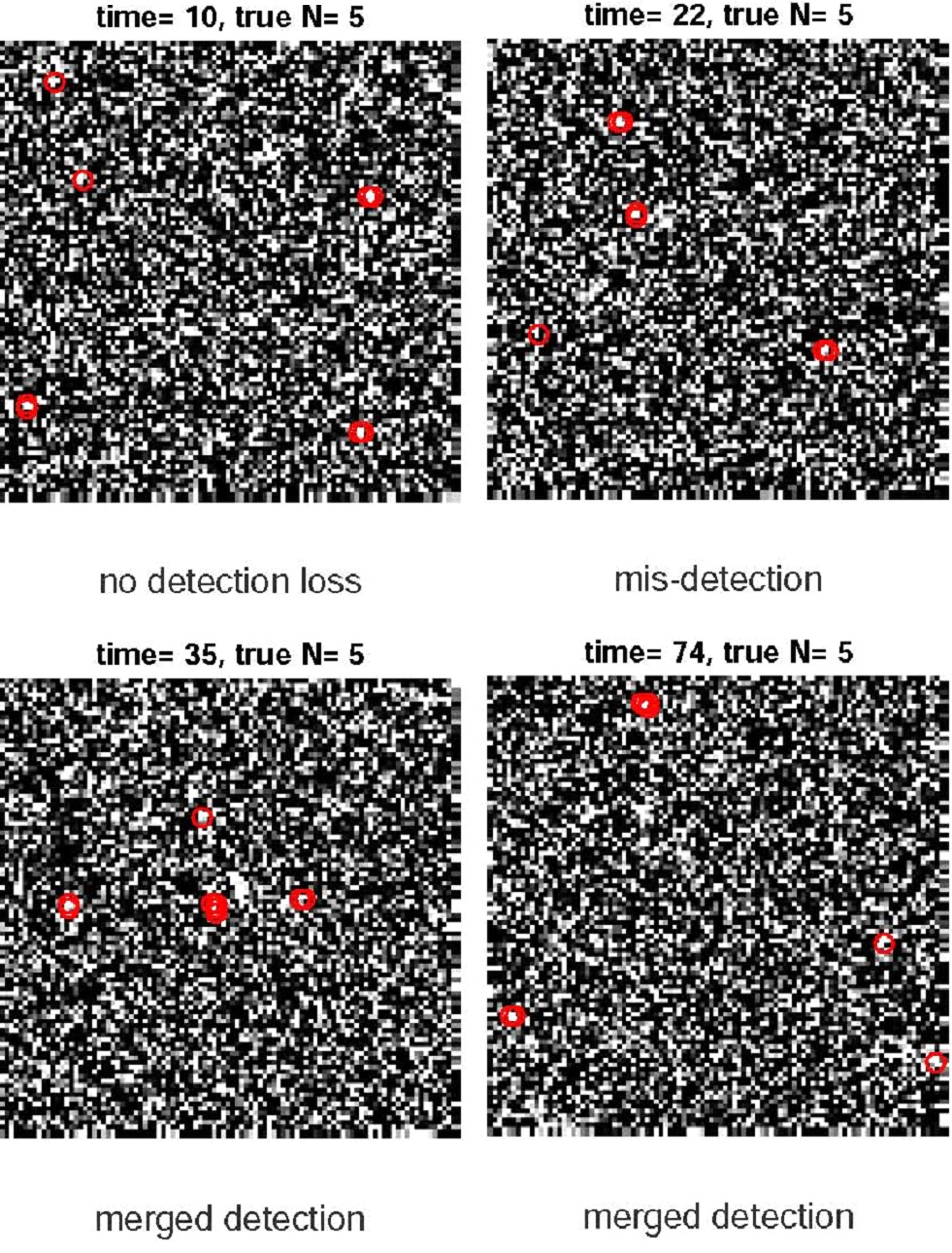} 
	\end{center}
	\caption{Snapshots of the image data (point detection is marked as circles)}
	\label{Fig: img_data}
\end{figure}
\begin{figure}[tbp]
	\begin{center}
		\includegraphics[height=.72\linewidth]{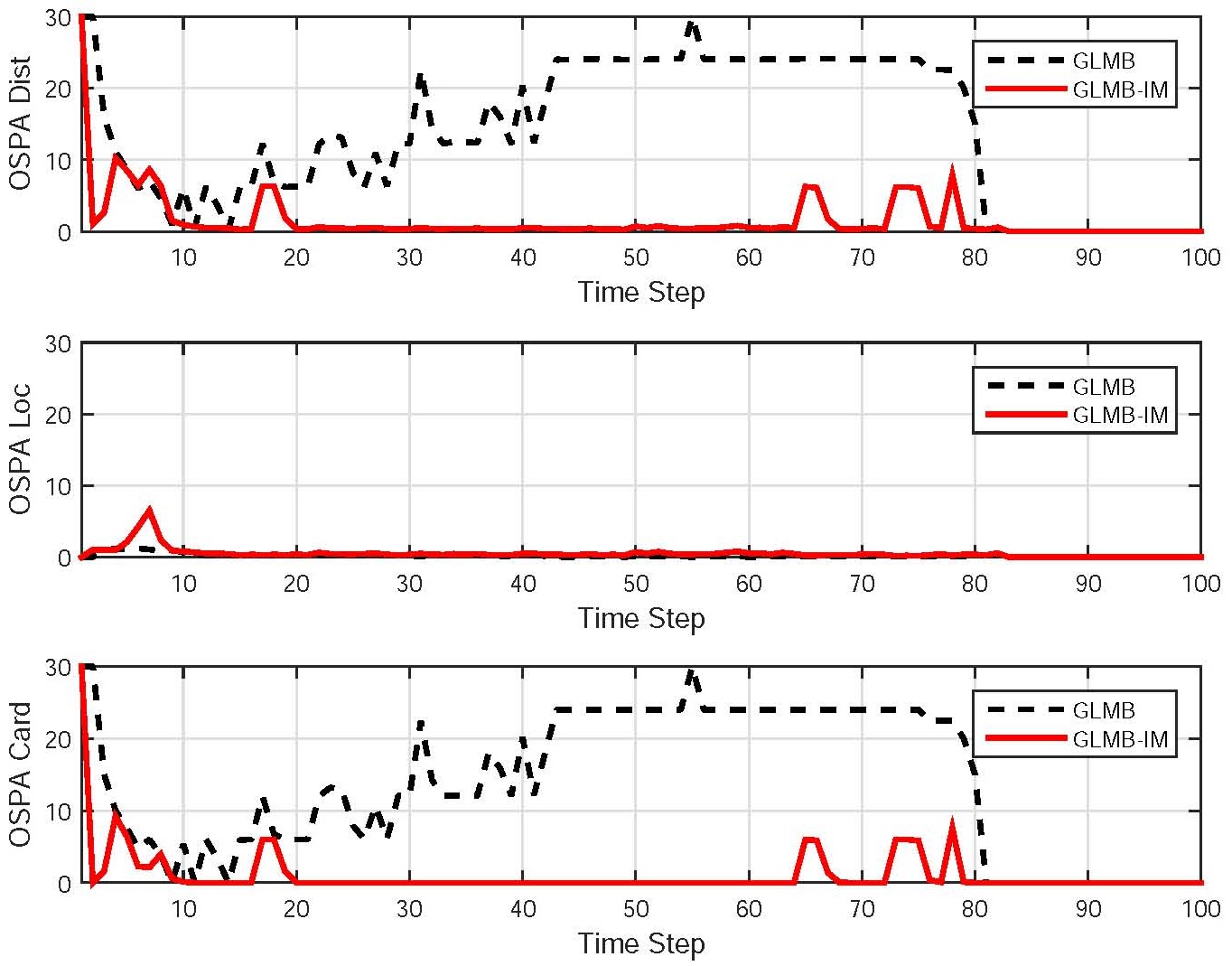} 
	\end{center}
	\caption{OSPA Error for three filters (first row: overall OSPA distance, Second row: localization error, third row: cardinality error)}
	\label{Fig: ospa}
\end{figure}

\subsubsection{Simulation scenario and comparison results}

The size of the surveillance area is 100 pixels by 100 pixels and the size
of the image cell is 1. Image data for the true tracks (shown in Fig. \ref%
{Fig: groundtruth}) is generated according the observation model (\ref%
{eq:pixel_return})-(\ref{eq:point_spread}). Sample snap shots of image
sequence are displayed in Fig. \ref{Fig: img_data} together with the true
number of objects and the description of detection results (by
hard-thresholding) for each snapshot. Fig. \ref{Fig: img_data} illustrates
that low SNR images are prone to mis-detections, and that merged detections
occur in mutual occlusions.\newline
\indent We compare the standard GLMB filter (GLMB) with the proposed GLMB for image
observation (GLMB-IM) (with time-dependent survival probability). The performance comparison is summarized in Fig. \ref{Fig: ospa}
with respect to OSPA errors \cite{OSPA} calculated over 100 Monte Carlo runs.\newline
\indent Note from Fig. \ref{Fig: ospa}, that the standard GLMB filter
quickly lost tracks due to the mis-detections from low SNR or merged
detections from object occlusions. On the other hand, the GLMB-IM filter keeps tracks due to the combination of proposed survival probability and effective measurement updates from the image data.

\subsection{Visual Tracking\label{subsec:visual_tracking}}

\subsubsection{Dataset and parameter settings}

In this subsection, we test the proposed MOT filter on publicly available
video data: the \textit{S2L1} sequence of the PET2009 dataset \cite{PETS09};
the \textit{BAHNHOF} and \textit{SUNNYDAY} sequence of the ETH dataset \cite%
{EssCVPR08}; and the \textit{TUD-Stadtmitte} sequence from \cite%
{AndrilukaCVPR10}. To benchmark the tracking performance against a number of
recent algorithms, we use published detection results and evaluation tools
from \cite{YangCVPR12}. The same motion model in the TBD example is used
with $\sigma _{v}=3$ pixels/frame (set by considering the maximum speed of
the object with regard to the frame rate).\newline
Remark: While the object's extent such as its bounding box \cite%
{Breitenstein11}, \cite{Milan_2014}, \cite{AndrilukaCVPR10}, can be included
in the object state, effective modeling of extent dynamics is application
dependent. In experiments we estimate an object's extent via the median
values of the x, y scale of the detections associated with existing tracks
in a given time window. \newline
{Remark}: Similar to single-object visual tracking filtering in \cite%
{Breitenstein11}, the predicted covariance for each track is capped to a
prescribed value to prevent it from exploding over time. \newline
\indent The RFS framework accommodates a time-varying birth model. In this
experiment, we use a birth model that consists of both static and dynamic
components. The static component is an LMB that describes expected locations
where objects are highly likely to appear e.g., the image border/footpaths
near the image border. The dynamic component is a time-varying LMB that
exploits measurements with weak associations (to existing tracks) to
describe highly likely object births at the next time frame \cite{Reuter14}.%
\newline
\indent The detection $z\in D(y_{k})$ of an object is obtained by a detector
based on aggregated channel features (ACF) \cite{DollarBMVC09ChnFtrs} and
the same point measurement model in the numerical example is used with $%
\Sigma =diag(5^{2},5^{2})$. The probability of detection $P_{D}$ is 0.98 and
the clutter rate is 5, i.e., an average of 5 clutter measurements per frame.
These parameters can be obtained from training data or learned on-the-fly in
the RFS framework as proposed in \cite{R_CPHD_Mahler}. For image
observations, we also used KCF method as in the numerical example but with
Histogram of Oriented Gradients (HOG) feature instead of raw pixel \cite%
{Henriques15}.\newline
\indent In the experiments, the maximum number of track hypotheses $%
H_{+}^{\max }$ is set to 200, and track estimates are obtained from the GLMB
filtering density via the LMB estimator described in Subsection \ref%
{subsec:VGLMB}. Note that when the LMB estimator terminates a track, the
GLMB filtering density still contains its existence probability and state
density (hence state estimate). This information is completely deleted only
when its existence probability is so negligible that all relevant GLMB
components are truncated. If not completely deleted, it is possible that due
to new evidence in the data at later time, a track's existence probability
becomes significant enough to be selected by multi-object estimator, leading
to track fragmentation. While this problem can be addressed in a principled
manner via multi-object smoothing, the GM-PHD smoother \cite{MVV12} is not
applicable and an implementation of the forward-backward GLMB smoother \cite%
{BVV16} is not yet available. Nonetheless, we can exploit the available
information on the terminated track from the GLMB density in previous frames
to recover missing state estimates.

\begin{table*}[tbp]
{\footnotesize \ }
\par
\begin{center}
{\footnotesize \ 
\begin{tabular}{|l|l|ccc|cccc|cc|}
\hline
\textbf{Dataset} & \textbf{Method} & \textbf{Recall} $\uparrow$ & \textbf{%
Precision} $\uparrow$ & \textbf{FPF} $\downarrow$ & \textbf{GT} & \textbf{MT}
$\uparrow$ & \textbf{PT} $\downarrow$ & \textbf{ML} $\downarrow$ & \textbf{%
Frag} $\downarrow$ & \textbf{IDS} $\downarrow$ \\ \hline\hline
& {GLMB-IM$\ast$} & 95.6 \% & 92.2 \% & 0.03 & 19 & 95 \% & 5 \% & 0.0 \% & 
23 & 5 \\ 
& {RMOT$\ast$} \cite{RMOT} & 95.6 \% & 95.4 \% & 0.05 & 19 & 94.7 \% & 5.3 \%
& 0.0 \% & 23 & 1 \\ 
PETS09-S2L1 & {StruckMOT$\ast$ \cite{Suna_2012}} & 97.2 \% & 93.7 \% & 0.38
& 19 & 94.7 \% & 5.3 \% & 0.0 \% & 19 & 4 \\ 
& {{PRIMPT} \cite{KuoCVPR11}} & 89.5 \% & 99.6 \% & 0.02 & 19 & 78.9 \% & 
21.1 \% & 0.0 \% & 23 & 1 \\ 
& {{CemTracker} \cite{Milan_2014}} & - & - & - & 19 & 94.7 \% & 5.3 \% & 0.0
\% & 15 & 22 \\ 
& {{KSP} \cite{Berclaz}} & - & - & - & 23 & 73.9 \% & 17.4 \% & 8.7 \% & 22
& 13 \\ 
& {{GeodesicTracker$\ast$} \cite{Possegger}} & - & - & - & 23 & 100 \% & 0 \%
& 0 \% & 16 & 9 \\ \hline\hline
& {GLMB-IM$\ast$} & 89.4 \% & 87.6 \% & 0.11 & 10 & 80.0 \% & 20.0 \% & 0.0
\% & 10 & 1 \\ 
& {{RMOT$\ast$} \cite{RMOT}} & 87.9 \% & 96.6 \% & 0.19 & 10 & 80.0 \% & 
20.0 \% & 0.0 \% & 7 & 6 \\ 
TUD-Stadtmitte & {StruckMOT$\ast$ \cite{Suna_2012}} & 87.3 \% & 95.4 \% & 
0.25 & 10 & 80.0 \% & 20.0 \% & 0.0 \% & 11 & 0 \\ 
& {{PRIMPT$\ast$} \cite{KuoCVPR11}} & 81.0 \% & 99.5 \% & 0.028 & 10 & 60.0
\% & 30.0 \% & 10.0 \% & 0 & 1 \\ 
& {{OnlineCRF} \cite{YangCVPR12}} & 87.0 \% & 96.7 \% & 0.18 & 10 & 70.0 \%
& 30.0 \% & 0.0 \% & 1 & 0 \\ 
& {{CemTracker} \cite{Milan_2014}} & - & - & - & 10 & 40.0 \% & 60.0 \% & 
0.0 \% & 13 & 15 \\ 
& {{KSP} \cite{Berclaz}} & - & - & - & 9 & 11.0 \% & 78.0 \% & 11.0 \% & 15
& 5 \\ \hline\hline
& {GLMB-IM$\ast$} & 80.1 \% & 85.6 \% & 0.98 & 124 & 62.4 \% & 32.6 \% & 5.0
\% & 70 & 20 \\ 
& {RMOT$\ast$} \cite{RMOT} & 81.5 \% & 86.3 \% & 0.98 & 124 & 67.7 \% & 27.4
\% & 4.8 \% & 38 & 40 \\ 
ETH & {StruckMOT$\ast$} \cite{Suna_2012} & 78.4 \% & 84.1 \% & 0.98 & 124 & 
62.7 \% & 29.6 \% & 7.7 \% & 72 & 5 \\ 
BAHNHOF and & {MOT-TBD$\ast$ \cite{PoiesiCVIU13}} & 78.7 \% & 85.5 \% & - & 
125 & 62.4 \% & 29.6 \% & 8.0 \% & 69 & 45 \\ 
SUNNYDAY & {{PRIMPT} \cite{KuoCVPR11}} & 76.8 \% & 86.6 \% & 0.89 & 125 & 
58.4 \% & 33.6 \% & 8.0 \% & 23 & 11 \\ 
& {{OnlineCRF} \cite{YangCVPR12}} & 79.0 \% & 85.0 \% & 0.64 & 125 & 68.0 \%
& 24.8 \% & 7.2 \% & 19 & 11 \\ 
& {{CemTracker} \cite{Milan_2014}} & 77.3 \% & 87.2 \% & - & 124 & 66.4 \% & 
25.4 \% & 8.2 \% & 69 & 57 \\ \hline
\end{tabular}
\\[0pt]
\vspace{3pt} \vspace{-10pt} }
\end{center}
\caption{Comparison of tracking performance with the state-of-the-art
trackers (Online methods are indicated by $\ast$)}
\label{tab:comparison_state_of_the_art}
\end{table*}
\begin{figure*}[t]
	\begin{center}
		$~$\includegraphics[height=.115\linewidth]{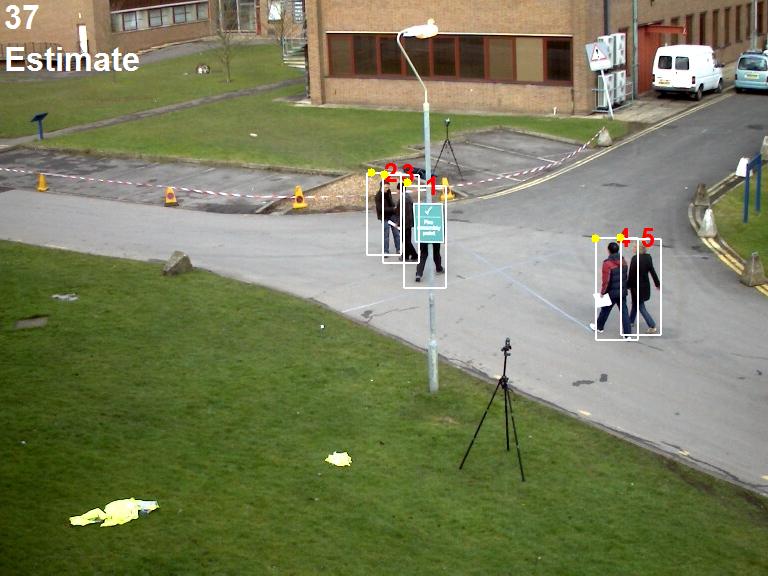}  %
		\includegraphics[height=.115\linewidth]{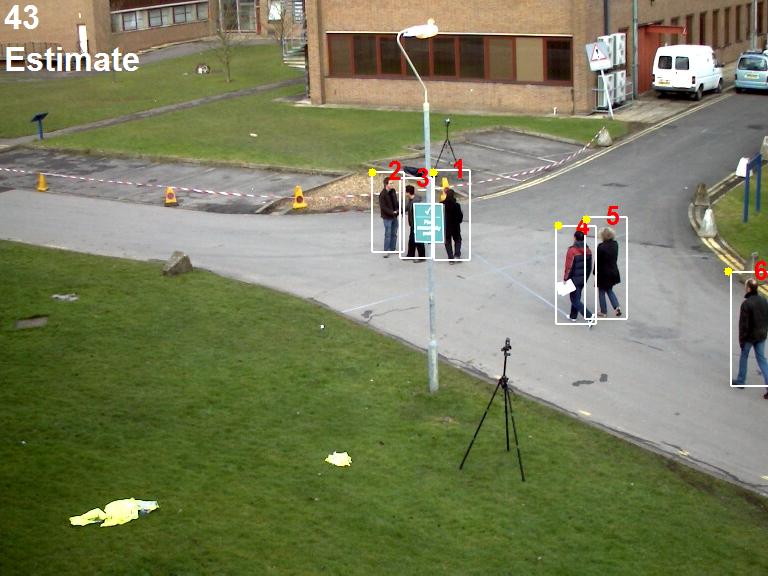}  %
		\includegraphics[height=.115\linewidth]{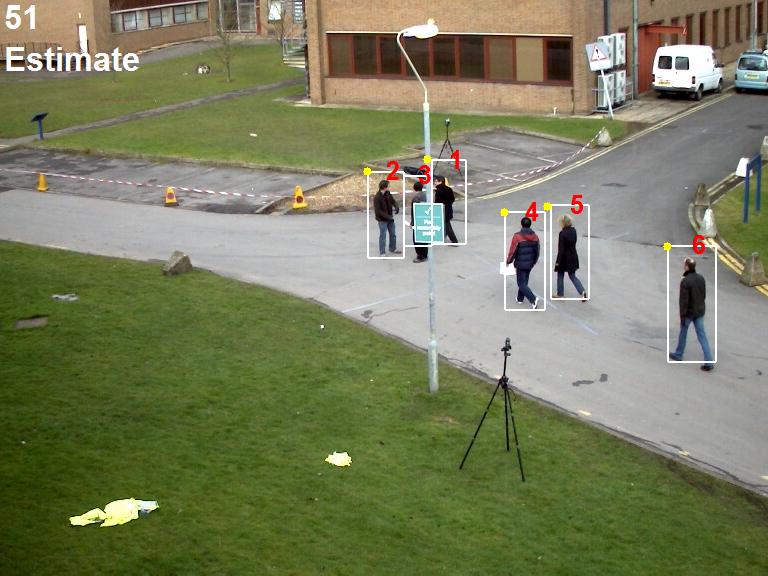}  %
		\includegraphics[height=.115\linewidth]{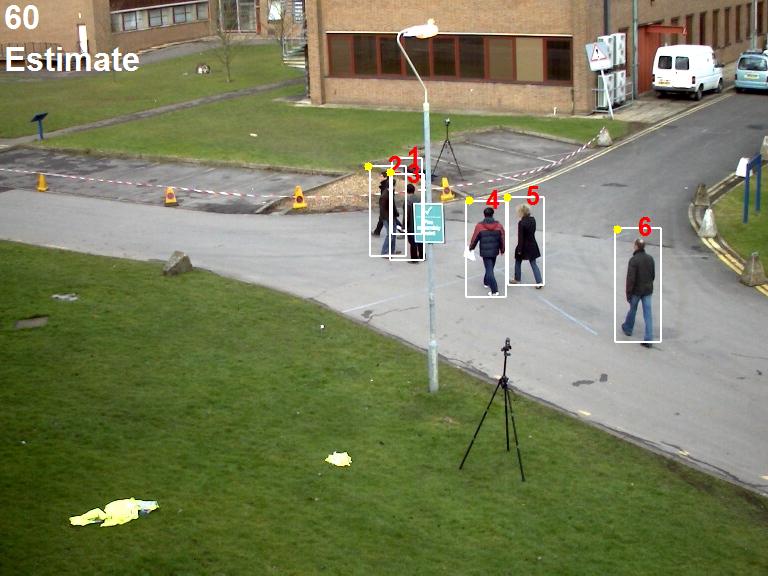}  %
		\includegraphics[height=.115\linewidth]{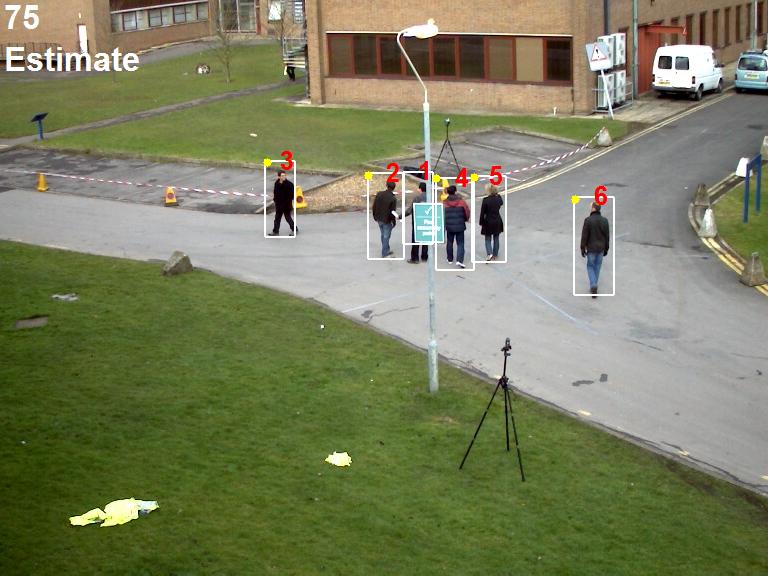}  %
		\includegraphics[height=.115\linewidth]{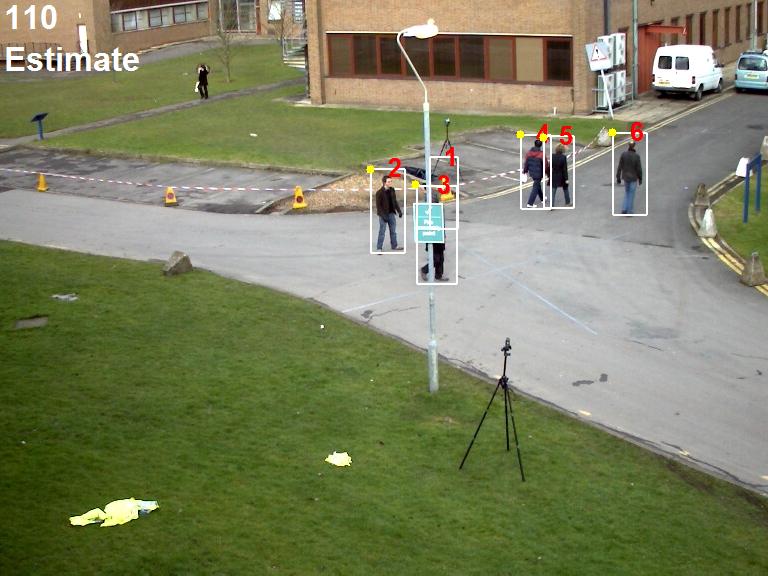}\\[0pt]
		$~$\includegraphics[height=.115\linewidth]{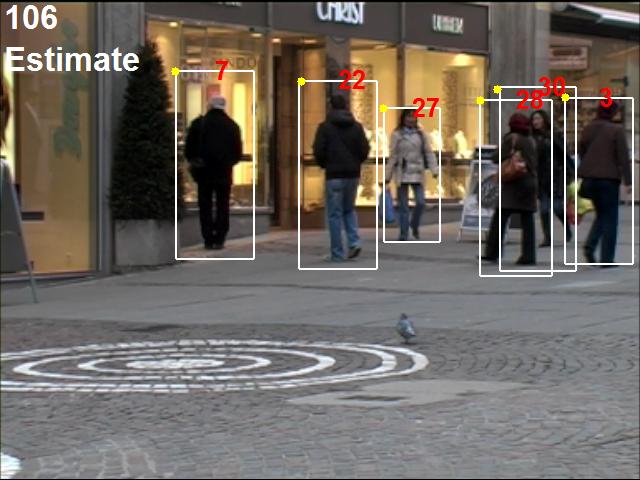}  %
		\includegraphics[height=.115\linewidth]{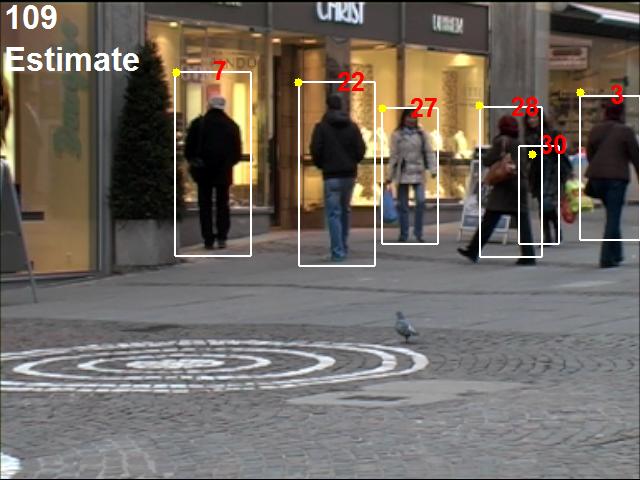}  %
		\includegraphics[height=.115\linewidth]{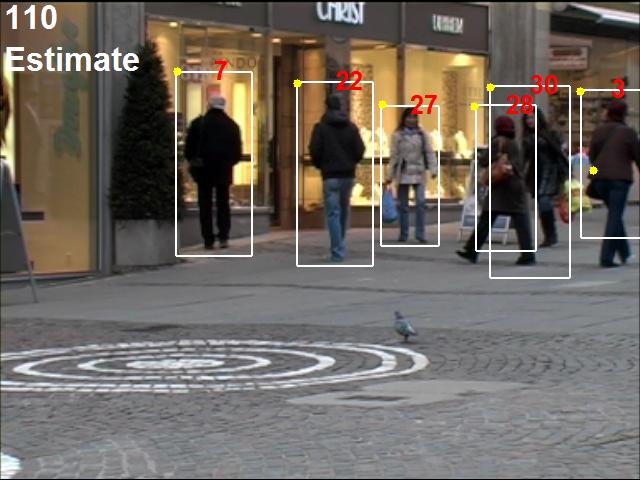}  %
		\includegraphics[height=.115\linewidth]{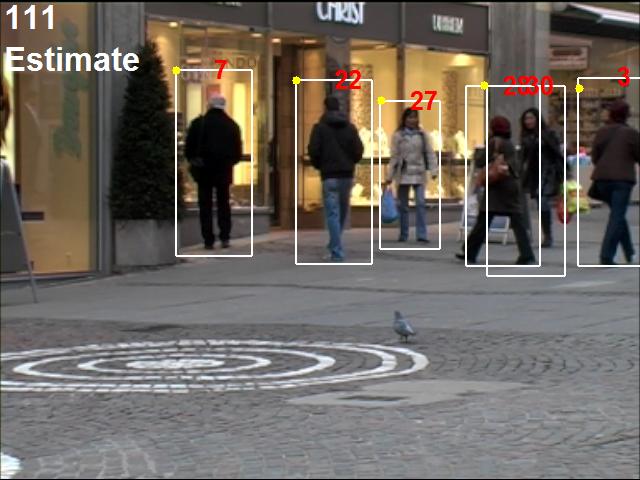}  %
		\includegraphics[height=.115\linewidth]{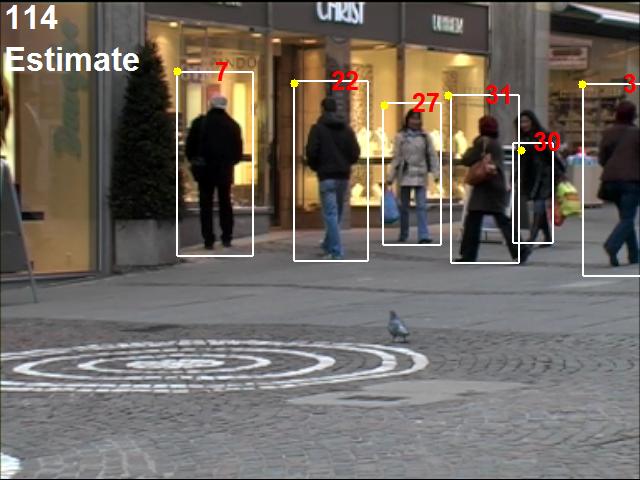}  %
		\includegraphics[height=.115\linewidth]{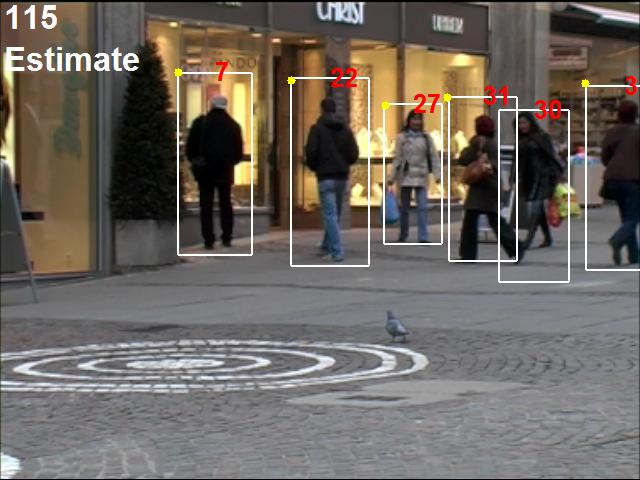}\\[0pt]
		$~$\includegraphics[height=.115\linewidth]{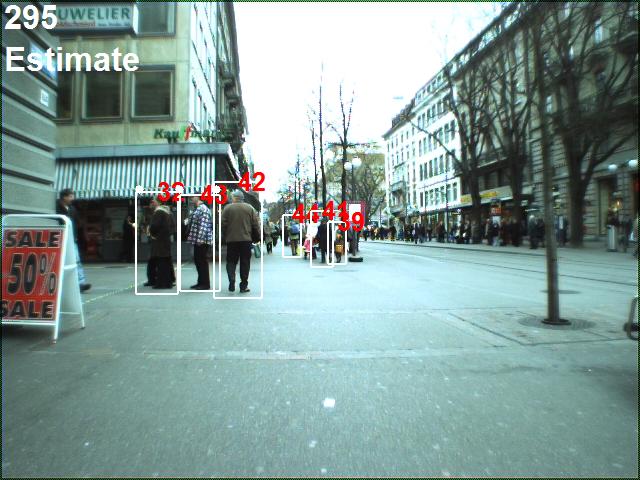}  %
		\includegraphics[height=.115\linewidth]{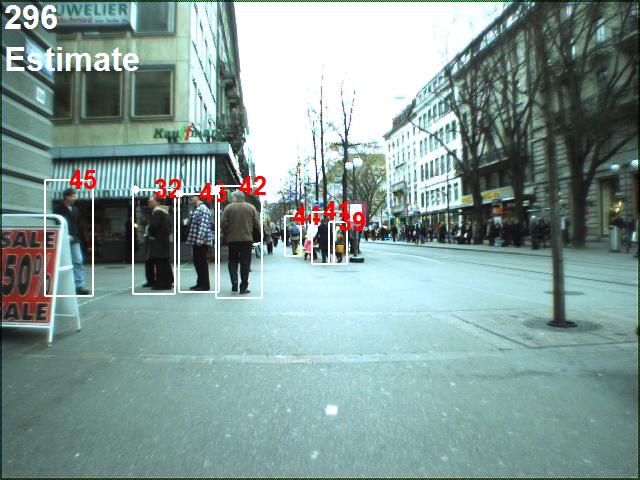}  %
		\includegraphics[height=.115\linewidth]{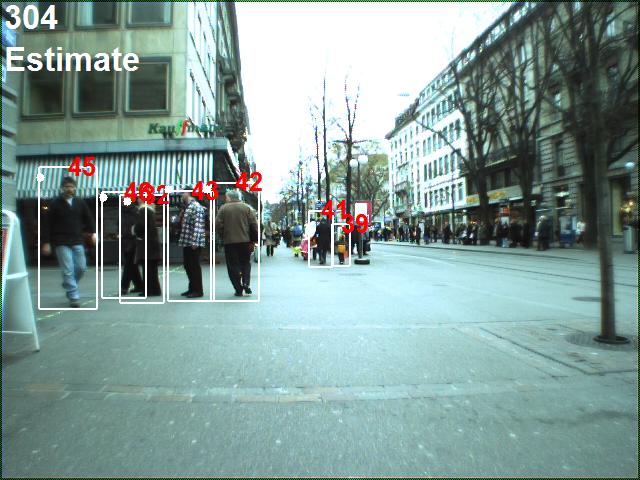}  %
		\includegraphics[height=.115\linewidth]{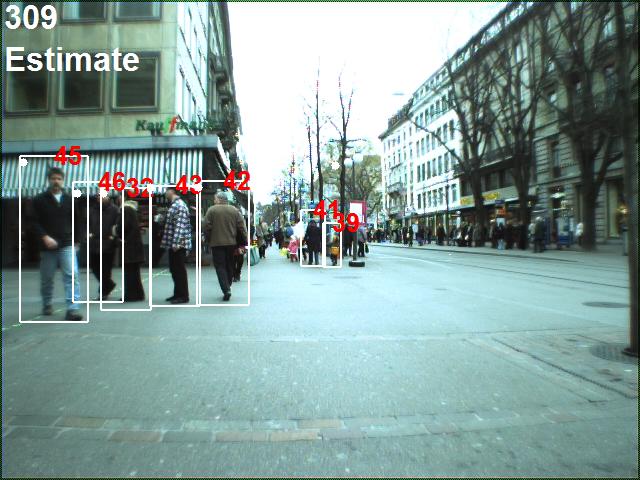}  %
		\includegraphics[height=.115\linewidth]{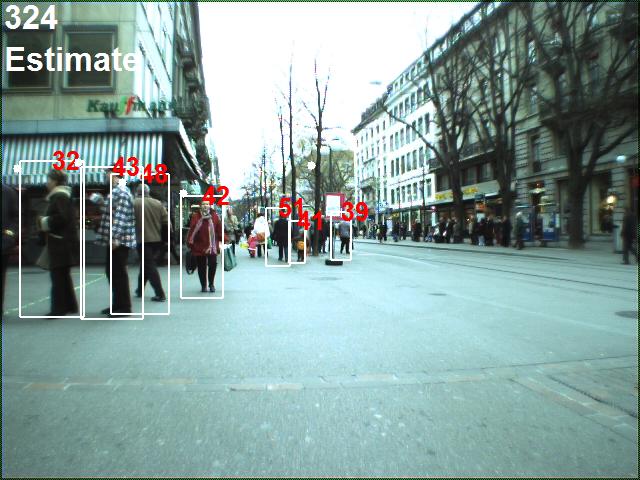}  %
		\includegraphics[height=.115\linewidth]{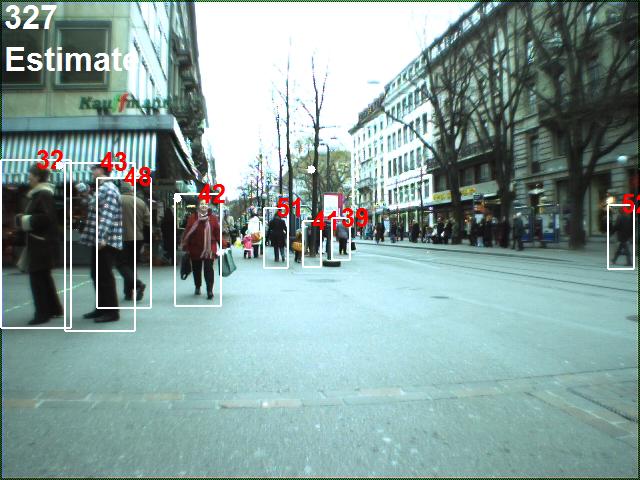}\\[0pt]
		$~$\includegraphics[height=.115\linewidth]{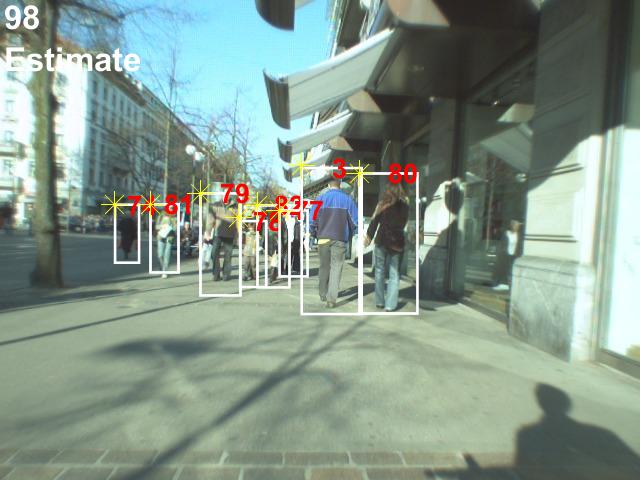}  %
		\includegraphics[height=.115\linewidth]{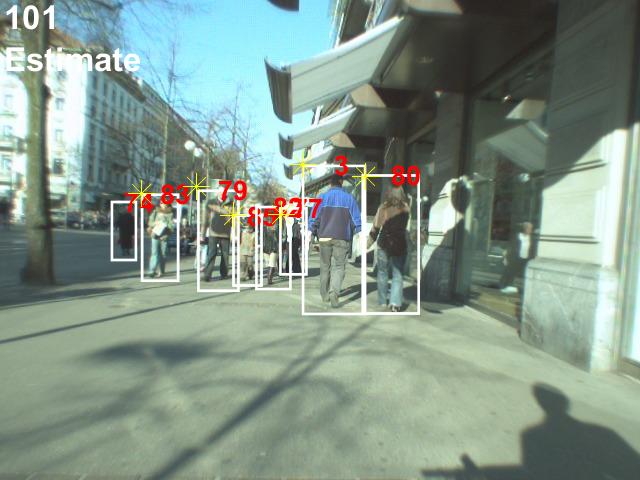}  %
		\includegraphics[height=.115\linewidth]{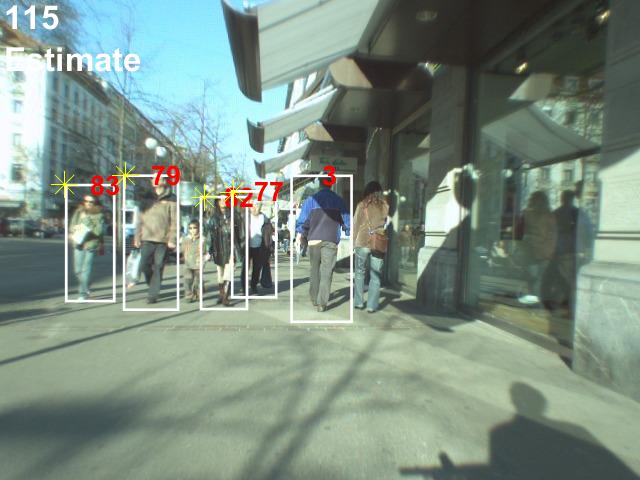}  %
		\includegraphics[height=.115\linewidth]{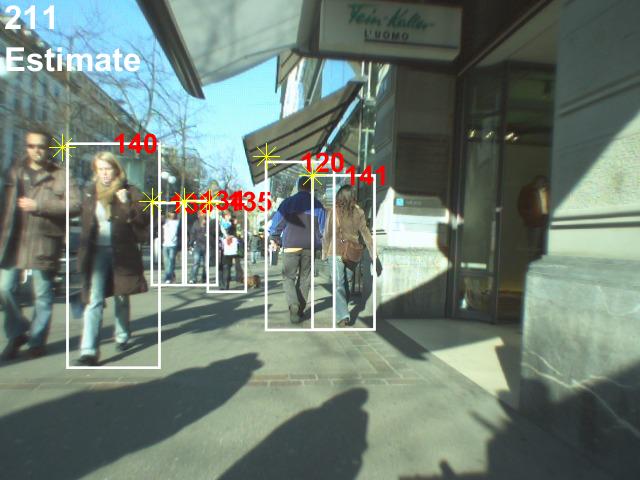}  %
		\includegraphics[height=.115\linewidth]{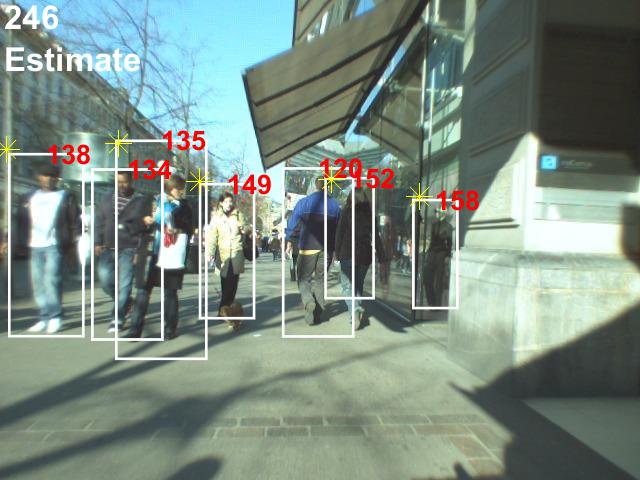}  %
		\includegraphics[height=.115\linewidth]{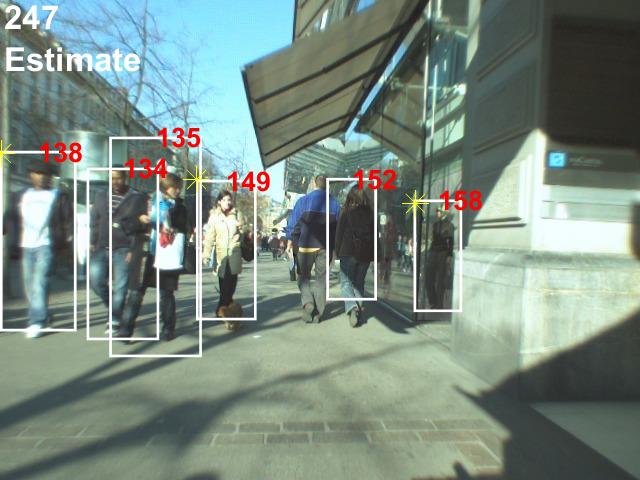}\\[0pt]

	\end{center}
	\caption{Selected frames of the tracking results}
	\label{Fig: visual_trackingResults}
\end{figure*}

\subsubsection{Quantitative performance analysis\label{subsec:Quantitative}}

The GLMB-IM filter tracking performance is benchmarked against offline-based
methods such as: StruckMOT \cite{Suna_2012}, PRIMPT \cite{KuoCVPR11},
CemTracker \cite{Milan_2014}, and KSP \cite{Berclaz}; and recent online
trackers: \cite{RMOT}, \cite{Possegger}. Also note that the online tracker 
\cite{Possegger} cannot be applied to the second and third sequences because
ground plane information is not available. We use well-known MOT performance
indices \cite{YangCVPR12} such as Recall (correctly tracked objects over
total ground truth), Precision (correctly tracked objects over total
tracking results), and false positives per frame (FPF). We also report the
number of identity switches (IDS) and the number of fragmentations (Frag).
Table \ref{tab:comparison_state_of_the_art} shows the ratio of tracks with
successfully tracked parts for more than 80\% (mostly tracked (MT)), less
than 20\% (mostly lost (ML)), or less than 80\% and more than 20\%
(partially tracked (PT)). The up (down) arrows in Table \ref%
{tab:comparison_state_of_the_art} mean that higher (lower) the values
indicate better performance.\newline
\indent As can be seen from Table \ref{tab:comparison_state_of_the_art}, the
GLMB-IM filter achieves the best or second best performance in important
indicators such as FPF, Recall and MT, amongst the online methods. For Frag
and IDS, the GLMB-IM filter is consistently in the top three performers in
average. More surprisingly, it has comparable accuracy with offline methods,
keeping in mind that it runs in near real-time with basic Matlab implementation (see
Table \ref{tab:averaged speed}). In summary the GLMB-IM filter offers
practical trade-offs between accuracy and speed for real-time applications.
Further, as briefly mentioned before, the GLMB-IM can be extended to offline
methods such as batch estimation or via smoothing techniques.\newline
\indent In the ETH sequences, the RMOT \cite{RMOT} shows slightly better
results because the proposed relative motion network model in RMOT is
especially tailored for handling of full occlusions in tracking scenarios of
group of people walking in the same directions. More fragmentation is
observed in the EHT sequences due to re-initialization of objects from the
measurement-driven birth model when they emerge from very long full
occlusions. Due to the generality of the framework, more sophisticated
motion models and other types of detections and appearance features can be
incorporated for further improvements. Selected frames of tracking results
for object occlusions are given in Fig. \ref{Fig: visual_trackingResults}.%
\newline
\begin{figure*}[tbp]
	\begin{center}
		\includegraphics[height=.20\linewidth]{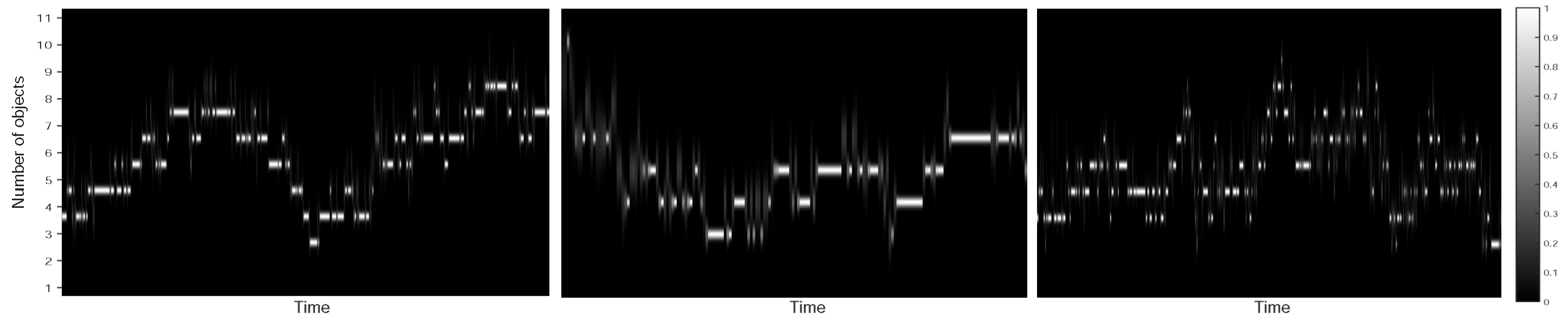} 
	\end{center}
	\caption{Cardinality distributions for three sequences left:\textit{S2L1},
		center:\textit{TUD}, right:\textit{ETH})}
	\label{Fig: Cardinality}
\end{figure*}
\indent The RFS approach also provides the probability distribution of the
current number of objects, i.e., cardinality distribution (\ref{eq:GLMBCard}%
) (which is not available in other tracking approaches). Fig. \ref{Fig:
Cardinality} shows the frame by frame cardinality distribution for the three
data sequences. \newline
\indent The tracking experiments with the proposed GLMB-IM filter are
implemented in MATLAB using single core (Intel i7 2.4GHz 5500) CPU laptop. A
comparison of tracking speed with other trackers (excluding the point
detection process) is summarized in Table \ref{tab:averaged speed}, which
shows an average (over all three experiments) of 20 fps for the GLMB-IM
filter (without code optimization). Hence, the GLMB-IM filter is very
well-suited for online applications considering further speed up can be
achieved using C++ and code optimization. Further, the salient feature of
the proposed GLMB-IM filter is its linear complexity with respect to the
number of detections \cite{VVH}.\newline
\indent It is important to note that the computation speeds in Table \ref%
{tab:averaged speed} only serves as a rough indication because all
implementations are dependent on the hardware platform, programming
language, code structure, test sequence scenarios, etc.

\section{Conclusion\label{sec:Conclusion}}

\begin{table}[tbp]
\begin{center}
{\small \ 
\begin{tabular}{|c|c|c|}
\hline
\textbf{Method} & \textbf{Average speed} & \textbf{Implementation} \\ 
\hline\hline
GLMB-IM & 20 fps & MATLAB \\ \hline
RMOT \cite{RMOT} & 3.5 fps & MATLAB \\ \hline
GeodesicTracker \cite{Possegger} & 11.2 fps & MATLAB \\ \hline
StruckMOT \cite{Suna_2012} & 15 fps & MATLAB \\ \hline
CemTracker \cite{Milan_2014} & 0.55 fps & MATLAB \\ \hline
OnlineCRF \cite{YangCVPR12} & 8 fps & C++ \\ \hline
\end{tabular}
}
\end{center}
\caption{Comparison of averaged speed }
\label{tab:averaged speed}
\end{table}

An efficient online MOT algorithm for video data that seamlessly integrates
state estimation, track management, clutter rejection, mis-detection and
occlusion handling into one single Bayesian recursion has been proposed.
Further, the proposed MOT algorithm exploits the advantages of both
detection-based and TBD approaches. In particular, it has the efficiency of
detection-based approach that avoids updating with the entire image, while
at the same time making use of information at the image level by using only
small regions of the image where mis-detected objects are expected. The
proposed algorithm has a linear complexity in the number of detections and
quadratic in the number of hypothesized tracks, making it suitable for
real-time computation. Moreover, experimental results on well-known datasets
indicated that proposed algorithm is competitive in tracking accuracy and
speed compared to data association based batch algorithms and recent online
algorithms.

% or


\appendices


\ifCLASSOPTIONcompsoc

\section*{Acknowledgments}

\else

\section*{Acknowledgment}

\fi

\ifCLASSOPTIONcaptionsoff
\newpage \fi



%\bibliographystyle{IEEEtran}

\bibliographystyle{IEEEtran}
\bibliography{bib}

%\begin{biography}[{\includegraphics[width=1in,height=1.25in,clip,keepaspectratio]{mshell}}]{Michael Shell}


%\vfill



\end{document}